\documentclass[11pt]{article}
\usepackage[a4paper,margin=1in]{geometry}
\usepackage{graphicx}
\usepackage{multirow}
\usepackage{amsmath,amssymb,amsfonts}
\usepackage{mathrsfs}
\usepackage{xcolor}
\usepackage{textcomp}
\usepackage{manyfoot}
\usepackage{booktabs}
\usepackage{algorithm}
\usepackage{algorithmic}
\usepackage{listings}
\usepackage{amsthm}

\usepackage[normalem]{ulem}
\usepackage{float}

\usepackage[english]{babel}
\usepackage[latin1]{inputenc}
\usepackage{enumerate}
\usepackage{color}
\usepackage[T1]{fontenc}
\usepackage{epstopdf}
\usepackage{subfigure}
\usepackage{dsfont}
\usepackage[T1]{fontenc}
\usepackage{amsmath}
\usepackage{mathtools}
\usepackage{amstext}
\usepackage{amssymb}
\usepackage{mathrsfs}
\usepackage{mathtools}
\usepackage{tikz}
\usetikzlibrary{arrows,positioning}
\usepackage{pifont}

\usepackage{tikz}
\usetikzlibrary{shapes,arrows}
\tikzstyle{block} = [draw, fill=white, rectangle,
minimum height=3em, minimum width=6em]
\tikzstyle{sum} = [draw, fill=white, circle, node distance=1cm]
\tikzstyle{input} = [coordinate]
\tikzstyle{output} = [coordinate]
\tikzstyle{pinstyle} = [pin edge={to-,thin,black}]

\usepackage{multicol}
\usepackage{lipsum}
\usepackage{tcolorbox}
\tcbuselibrary{skins,breakable}
\usetikzlibrary{shadings,shadows}

\newcommand{\R}{\mathbb{R}}

\newcommand{\cl}{\mathop{\rm cl}}

\newcommand{\cspan}{\mathop{\rm span}}
\newcommand{\softmax}{\mathop{\rm softmax}}

\newcommand{\Law}{\mathop{\rm Law}}

\newcommand{\ee}{\mathop{\rm e}}

\newcommand{\argmax}{\mathop{\rm arg\, max}}
\newcommand{\argmin}{\mathop{\rm arg\, min}}

\newcommand{\bmu}{{\boldsymbol \mu}}

\newcommand{\bpi}{{\boldsymbol \pi}}

\newcommand{\blambda}{{\boldsymbol \lambda}}
\newcommand{\btheta}{{\boldsymbol \theta}}

\newcommand{\bh}{{\bf h}}

\newcommand{\sX}{{\mathsf X}}

\newcommand{\sA}{{\mathsf A}}

\newcommand{\sZ}{{\mathsf Z}}

\newcommand{\A}{{\mathcal A}}
\newcommand{\cH}{{\mathcal H}}
\newcommand{\cL}{{\mathcal L}}

\newcommand{\D}{{\mathcal D}}

\newcommand{\V}{{\mathcal V}}
\newcommand{\Pnew}{{\mathcal P}}

\newcommand{\cG}{{\mathcal G}}

\newcommand{\tv}[1]{\textcolor{black}{#1}} 
\newcommand{\tb}[1]{\textcolor{black}{#1}} 
\newcommand{\lk}[1]{\textcolor{black}{#1}} 

\newtheorem{proposition}{Proposition}[section]
\newtheorem{definition}{Definition}[section]
\newtheorem{theorem}{Theorem}[section]
\newtheorem{assumption}{Assumption}[section]

\newtheorem{remark}{Remark}

\allowdisplaybreaks

\usepackage{cite}
\usepackage{amsmath,amssymb,amsfonts}
\usepackage{algorithmic}
\usepackage{graphicx}
\usepackage{algorithm,algorithmic}
\usepackage{hyperref}
\hypersetup{hidelinks=true}
\usepackage{textcomp}
\def\BibTeX{{\rm B\kern-.05em{\sc i\kern-.025em b}\kern-.08em
    T\kern-.1667em\lower.7ex\hbox{E}\kern-.125emX}}
\begin{document}
\title{Kernel Based Maximum Entropy Inverse Reinforcement Learning for Mean-Field Games}
\date{}
\author{Berkay Anahtarci, Can Deha Kariksiz, and Naci Saldi\thanks{Berkay Anahtarci and Can Deha Kariksiz are with the Department of Mathematical Engineering, \"{O}zye\u{g}in University, {I}stanbul, Turkey (e-mail: berkay.anahtarci@ozyegin.edu.tr, deha.kariksiz@ozyegin.edu.tr). Naci Saldi is with the Department of Mathematics at Bilkent University, Ankara, Turkey (e-mail: naci.saldi@bilkent.edu.tr). This work was supported by the Scientific and Technological Research Council of Turkey (TUBITAK), under Grant no: 1001-124F134.}}\maketitle 

\begin{abstract}
\tb{We consider the maximum causal entropy inverse reinforcement learning (IRL) problem for infinite-horizon stationary mean-field games (MFG), in which we model the unknown reward function within a reproducing kernel Hilbert space (RKHS). This allows the inference of rich and potentially nonlinear reward structures directly from expert demonstrations, in contrast to most existing approaches for MFGs that typically restrict the reward to a linear combination of a fixed finite set of basis functions and rely on finite-horizon formulations. We introduce a Lagrangian relaxation that enables us to reformulate the problem as an unconstrained log-likelihood maximization and obtain a solution via a gradient ascent algorithm. To establish the theoretical consistency of the algorithm, we prove the smoothness of the log-likelihood objective through the Fr\'echet differentiability of the related soft Bellman operators with respect to the parameters in the RKHS. To illustrate the practical advantages of the RKHS formulation, we validate our framework on a mean-field traffic routing game exhibiting state-dependent preference reversal, where the kernel-based method reduces policy recovery error by over an order of magnitude compared to a linear reward baseline with a comparable parameter count. Furthermore, we extend the framework to the finite-horizon non-stationary setting. We demonstrate that the log-likelihood reformulation is structurally unavailable in this regime and instead develop an alternative gradient descent algorithm on the convex dual via Danskin's theorem, establishing smoothness and convergence guarantees.}

\end{abstract}

\noindent\textbf{Keywords}: Discounted reward, inverse reinforcement learning, maximum causal entropy, mean-field games, reproducing kernel Hilbert spaces.

\section{Introduction}
\label{intro}

Mean-field games (MFGs) provide a framework for analyzing strategic interactions in large populations of agents, where the agents influence each other through a mean-field term that captures the aggregate distribution of the population's states. The infinite-population limit is typically considered \cite{HuMaCa06, LaLi07, BeFrPh13, CaDe13}, as it yields a tractable characterization of equilibria and serves as an accurate approximation for systems with many agents. This approach has diverse applications such as crowd motion \cite{7874135}, opinion dynamics \cite{BaTeBa2016}, and traffic routing \cite{9061051}.

In stationary MFGs, the collective behavior of other agents \cite{WeBeRo05} is characterized through a time-invariant distribution that leads to a Markov Decision Process (MDP) constrained by the state's stationary distribution. The equilibrium concept in this case referred to as the stationary mean-field equilibrium (MFE), involves a policy and a distribution satisfying the \lk{Nash certainty equivalence principle} \cite{HuMaCa06}, where the policy is optimal with respect to a fixed stationary distribution representing the population behavior, and when adopted by a representative agent, it induces a stationary state distribution that coincides with the assumed mean-field distribution. Under mild regularity conditions, the existence of a stationary MFE can be established via Kakutani's fixed point theorem. Moreover, in the limit of a large number of agents, the policy corresponding to a stationary MFE serves as an approximate Nash equilibrium \lk{for the corresponding finite-agent game} \cite{AdJoWe15}. Recent years have seen increasing interest in learning algorithms for stationary MFGs. For instance, \cite{YiMeMeSh14} introduced a learning method for mean-field oscillator games, proving convergence to an approximate Nash equilibrium.

Traditional MFG approaches compute the MFE when well-defined reward functions are provided, often employing forward reinforcement learning (RL) techniques \cite{LaPePeGiMuElGePi24}. \tb{
In many practically relevant MFG applications, however, this assumption is unrealistic: agent objectives are unobservable, heterogeneous, and too complex to be specified explicitly.
A canonical example is traffic routing \cite{9061051}, where drivers' route choices reflect individual trade-offs among travel time, fuel consumption, tolls, route familiarity, and risk preferences. These factors vary across the population and cannot be reliably elicited, yet the resulting traffic flow often converges to an observable equilibrium. 
In such settings, IRL enables inference of latent reward structures directly from observed equilibrium behavior. Standard RL formulations may be inadequate because they presuppose a known reward function, which is precisely the unknown quantity of interest. By recovering implicit objectives consistent with observed MFEs, IRL facilitates equilibrium prediction and control design without explicit reward modeling, and improves generalization beyond the observed data \cite{AdCoBe22}.
}



There has been recent progress in applying IRL methods for finding equilibria in multi-agent games \cite{KOPF201714902,LIAN2022110524,9815022}. Several papers also address the IRL problem within the context of MFGs. In \cite{YaYeTrXuZh18}, the authors reduce an MFG to an MDP in a fully-cooperative setting where all agents share the same societal reward, and employ the principle of maximum entropy to solve the corresponding IRL problem. In the case of a decentralized information structure and a non-cooperative setting, \cite{YaLiLiHu22} formulates the IRL problem for MFGs and considers a maximum-margin approach. 
\tb{
In our non-cooperative formulation, the mean-field equilibrium $(\pi_E, \mu_E)$ represents a Nash equilibrium in the infinite-population limit: no individual agent can improve its expected reward by unilaterally deviating from $\pi_E$ when the population distribution is $\mu_E$. The IRL problem then seeks to recover the individual reward function $r$ that rationalizes this observed equilibrium behavior.} \cite{ChZhLiWi23} proposes a mean-field adversarial IRL method that assumes expert demonstrations are generated from an entropy-regularized MFE, integrating concepts from decentralized IRL for MFGs, maximum entropy IRL, and generative adversarial learning. The uniqueness of the resulting MFE is established in their approach through a variational formulation aligned with the maximum likelihood principle. 
\tb{
More recently, \cite{chen2024meta} develops a meta-IRL framework for MFGs using probabilistic context variables to manage heterogeneous tasks, highlighting IRL's scalability in multi-agent environments. However, these MFG-IRL approaches \cite{YaYeTrXuZh18, YaLiLiHu22, ChZhLiWi23, chen2024meta} are limited to finite-horizon formulations, which yield convex optimization problems, and rely on either the classical maximum entropy principle or maximum-margin methods with linear reward parameterizations.}

\tb{
Beyond the MFG-specific literature, IRL has been studied in multi-agent settings more broadly. In Markov games, \cite{LinAB19} develops IRL methods based on inverse Nash equilibrium computation for finite-player settings, though their framework addresses single-agent MDPs rather than the mean-field regime. \cite{chandra2025}  extends maximum entropy IRL to unstructured interactive crowds but does not exploit mean-field structure for scalability. These approaches do not extend naturally to the large-population regime where MFG approximations become essential.}
\tb{
Separately, kernel methods have been employed in forward RL to provide flexible function approximation with theoretical guarantees: \cite{TaylorP09} introduces kernel-based value function approximation in MDPs, and \cite{vakili2024kernelbased} develops RKHS methods with order-optimal regret bounds. 
Our work draws on the representational power of kernel methods while addressing the inverse problem in the MFG setting.}


\tb{The classical maximum entropy framework, as employed in \cite{ChZhLiWi23}, is generally inapplicable in infinite-horizon settings, since the distribution over trajectories induced by the state-action process becomes ill-defined on the path space. To overcome a similar issue in infinite-horizon MDPs, \cite{ZhBlBa18} proposes the maximum causal entropy principle, which extends the maximum entropy framework by ensuring well-defined trajectory distributions through causality constraints.} \tv{Our prior work \cite{AnKaSa25} investigates the maximum causal entropy IRL problem for discrete-time stationary MFGs, extending the framework in \cite{ZhBlBa18} by modeling the unknown reward function as a linear combination of fixed basis functions that leads to a nonconvex optimization problem over policies, reformulate this problem as a convex optimization over state-action occupation measures by leveraging the linear programming characterization of MDPs, and develop a gradient descent algorithm with guaranteed convergence.}

\tb{While \cite{LinAB19} and our prior work \cite{AnKaSa25} do address infinite-horizon settings, and while kernel-based methods \cite{TaylorP09}, \cite{vakili2024kernelbased} and log-likelihood formulations \cite{GlTo22, ZiBaDe10, ZiBaDe13} remain restricted to finite-horizon MDPs, our contribution advances along four distinct axes: (i) stationary mean-field coupling: we incorporate population-level equilibrium constraints through a fixed-point condition on the stationary mean-field distribution, which is absent in single-agent formulations such as \cite{ZhBlBa18}; (ii) log-likelihood formulation: this perspective is absent in \cite{LinAB19,ZhBlBa18,AnKaSa25}, yet it may carry important practical implications given the extensive literature on maximum log-likelihood methods; (iii) gradient-based optimization under RKHS rewards: modeling the reward function within a reproducing kernel Hilbert space, we enable inference of complex nonlinear reward structures, moving beyond the linear parameterizations in \cite{AnKaSa25} and requiring a fundamentally different optimization approach via Lagrangian relaxation rather than convex programming; and (iv) extension to the non-stationary finite-horizon regime: we prove that the log-likelihood reformulation is structurally unavailable in this setting (Theorem 6.1) and develop an alternative convex dual minimization via Danskin's theorem, with $L$-smoothness (Proposition 6.2) and convergence guarantees (Theorem 6.2).}

\tb{The infinite-horizon stationary formulation is particularly well suited for problems characterized by long-run equilibrium behavior, since it induces time-invariant Bellman equations that are computationally tractable and analytically convenient.}
Extending to this regime requires proving new regularity results, namely the  Fr\'{e}chet differentiability of soft Bellman operators (Theorem~4.1) and $L$-smoothness of the objective (Proposition~5.1), which constitute our main technical contributions and have no analogues in the finite-horizon or linear-reward settings. 

Finally, we validate our method on a mean-field traffic routing game exhibiting state-dependent preference reversal, demonstrating that the kernel-based method significantly reduces policy recovery error compared to a linear reward baseline with a comparable parameter count.

\section{Preliminaries}\label{me-irl-mfg}
A discrete-time stationary MFG is defined by \((\sX, \sA, p, r, \beta)\), where \(\sX\) and \(\sA\) are finite state and action spaces, \(p: \sX \times \sA \times \Pnew(\sX) \to \Pnew(\sX)\) is the continuous transition probability, \(r: \sX \times \sA \times \Pnew(\sX) \to [0,\infty)\) is the continuous one-stage reward function, and \(\beta \in (0,1)\) is the discount factor. A state-measure \(\mu \in \Pnew(\sX)\) represents the population's stationary distribution, assumed to be constant across time. Given the state \(x(t)\), the action \(a(t)\), and the state-measure \(\mu\), the agent receives the reward \(r(x(t),a(t),\mu)\), and the next state evolves as \(x(t+1) \sim p(\cdot|x(t),a(t),\mu)\).


To fully describe the model dynamics, we need to specify how the agent chooses its actions. To that end, a policy $\pi$ is a conditional distribution on $\sA$ given $\sX$; that is, $\pi: \sX \rightarrow \Pnew(\sA)$. Let $\Pi$ denote the set of all policies. 


For a fixed $\mu$, the infinite-horizon discounted reward function of any policy $\pi$ is given by
\begin{align}
J_{\mu}(\pi,\mu_0) &= E^{\pi,\mu_0}\biggl[ \sum_{t=0}^{\infty} \beta^t r(x(t),a(t),\mu) \biggr], \nonumber 
\end{align}
where $\beta \in (0,1)$ is the discount factor and $x(0) \sim \mu_0$ is the initial state distribution.

\tb{
In the stationary MFG setting considered throughout this paper, we assume \(\mu_0=\mu\), so that the system is initialized at the stationary distribution. Consequently, the state process is stationary from \(t=0\), and \(E^{\pi,\mu}\) denotes expectation under this steady-state regime.
}

To formally define the concept of equilibrium in this MFG model, we introduce two set-valued mappings. Let $2^{S}$ denote the collection of all subsets of a set $S$. Then, the mapping $\Psi : \Pnew(\sX) \rightarrow 2^{\Pi}$ defined by
$$\Psi(\mu) = \{\hat{\pi} \in \Pi: J_{\mu}(\hat{\pi},\mu) = \sup_{\pi} J_{\mu}(\pi,\mu) \}$$
represents the optimal policies for a specified $\mu$.
On the other hand, $\Lambda : \Pi \to 2^{\Pnew(\sX)}$ maps any policy $\pi \in \Pi$ to the set of all state-measures $\mu_{\pi}$ that are invariant distributions of the transition probability $p(\,\cdot\,|x,\pi,\mu_{\pi})$. That is, $\mu_{\pi} \in \Lambda(\pi)$ if
\begin{align}
\mu_{\pi}(\,\cdot\,) = \sum_{x \in \sX} p(\,\cdot\,|x,a,\mu_{\pi}) \, \pi(a|x) \, \mu_{\pi}(x). \nonumber
\end{align}

\begin{definition}
A pair $(\pi_*,\mu_*) \in \Pi \times \Pnew(\sX)$ is called a mean-field equilibrium if $\pi_* \in \Psi(\mu_*)$ and $\mu_* \in \Lambda(\pi_*)$. That is, $\pi^*$ is  optimal with respect to the population distribution $\mu^*$, and $\mu^*$ remains invariant under the policy $\pi^*$.
\end{definition}

The core objective of IRL is to infer an underlying reward function from observed expert demonstrations enabling the derivation of robust policies that mimic or generalize expert behavior. To address the inherent ill-posedness of this inverse problem, it is essential to impose structural constraints on the space of admissible reward functions.

In this work, we assume the unknown (latent) reward function resides in a separable RKHS $\cH \subset \R^{\sZ}$, which is induced by a positive semi-definite kernel
$$
k: \sZ \times \sZ \rightarrow \R,
$$
where $\sZ \triangleq \sX \times \sA \times \Pnew(\sX)$. To simplify the notation, we define the associated feature map $\Phi$ by 
$$
\Phi: \sZ \ni z \mapsto \Phi(z) \triangleq k(\cdot,z) \in \cH.
$$
Consequently, the RKHS $\mathcal{H}$ can be characterized as the closure of the linear span of the feature evaluations. That is,
$$
\cH = \cl \cspan\{\Phi(z): z \in \sZ\},
$$
where the closure is taken with respect to the norm topology induced by the inner product
$$
\left \langle \sum_{i=1}^n \alpha_i \, \Phi(z_i), \sum_{j=1}^m \gamma_j \, \Phi(y_i) \right \rangle_{\cH} \triangleq \sum_{i=1,j=1}^{n,m} \alpha_i \, \gamma_j \, k(z_i,y_j).
$$
This inner product satisfies the fundamental reproducing property: for any function $f \in \mathcal{H}$ and any point $z \in \sZ$,
$$
f(z) = \langle f, \Phi(z) \rangle_{\cH}.
$$
A direct consequence of this property is the identity
$$
k(z,y) = \langle \Phi(z),\Phi(y) \rangle_{\cH},
$$
which reveals that the kernel $k$ computes the inner product between feature mappings in $\mathcal{H}$, often called the \textit{kernel trick}.

We assume the unknown reward function $r$ is an element of $\mathcal{H}$. Therefore, it can be represented (or arbitrarily well approximated) in the form
\[
r(\cdot) = \sum_{i=1}^n \alpha_i \, \Phi(z_i).
\]
Applying the reproducing property to this representation yields an equivalent, explicit form for evaluating the reward 
\[
r(z) = \sum_{i=1}^n \alpha_i \, \langle \Phi(z), \Phi(z_i) \rangle_{\mathcal{H}}.
\]
This formulation provides a linear parameterization of the unknown reward function within the feature space defined by $\Phi$. For further details on the fundamentals of RKHS theory, we refer the reader to the comprehensive introduction presented in \cite{PaRa16}.

In the IRL setting,  the expert provides a dataset $\mathcal{D}$ consisting of $d$ independent trajectories
$$
\D=\left\{\left(x_i(t),a_i(t)\right)_{t=0}^{T_i}\right\}_{i=1}^d,
$$
where each trajectory is generated according to a mean-field equilibrium $(\pi_E, \mu_E)$. The variable $T_i$ denotes the length of the $i$-th trajectory. Since $\mu_E$ is an invariant distribution of the transition probability $p(\cdot|x,\pi_E,\mu_E)$ under policy $\pi_E$ when the mean-field term in state dynamics is $\mu_E$, standard ergodic theory implies that, under mild conditions (e.g., irreducibility and aperiodicity), the empirical state visitation frequency converges almost surely to the stationary distribution. That is,
$$
\lim_{T_i\rightarrow \infty} \frac{1}{T_i} \sum_{t=0}^{T_i}  1_{\{x_i(t)=x\}}  = \mu_E(x) ~~\forall x \in \sX.
$$
Consequently, for a sufficiently large horizon $T_i$, the expert's mean-field distribution $\mu_E$ can be empirically approximated by averaging across all trajectories
$$
\frac{1}{d}  \sum_{i=1}^d \left(\frac{1}{T_i} \sum_{t=0}^{T_i}  1_{\{x_i(t)=x\}} \right) \simeq \mu_E(x) ~~ \forall x \in \sX.
$$
Furthermore, for a sufficiently large number of demonstrations $d$, the discounted feature expectation can be empirically approximated as
\begin{equation*}
\frac{1}{d} \sum_{i=1}^d \left( \sum_{t=0}^{T_i} \beta^t \, \Phi(x_i(t),a_i(t),\hat \mu_E) \right) \simeq \langle \Phi \rangle_{\pi_E,\mu_E},
\end{equation*}
where $\hat{\mu}_E$ is the empirical estimate of $\mu_E$ and the true feature expectation is defined, in the Bochner integral sense, by
$$\langle \Phi \rangle_{\pi_E,\mu_E} \coloneqq E^{\pi_E,\mu_E}\left[\sum_{t=0}^{\infty} \beta^t \,\Phi(x(t),a(t),\mu_E)\right] \in \cH.$$
Based on this empirical convergence, we adopt the following standard assumption for the remainder of this paper, which is common in the IRL literature.

\begin{assumption}
The following quantities are known: (i) the mean-field distribution $\mu_E$, and (ii) the discounted feature expectation $\langle \Phi \rangle_{\pi_E, \mu_E}$ under the expert policy $\pi_E$.
\end{assumption}

\tb{\begin{remark}
The mean-field distribution $\mu_E$ represents aggregate population behavior, which can be directly estimated from expert trajectory data. By the ergodicity arguments above, empirical state frequencies converge to $\mu_E$ as the trajectory length or the number of demonstrations increases. Regarding the existence of MFE, since $\mathsf{X}$ and $\mathsf{A}$ are finite and the transition probability $p$ and reward function $r$ are continuous in the mean-field argument, existence of a stationary MFE follows from Kakutani's fixed point theorem~\cite{AdJoWe15}. Our framework does not require uniqueness---we only require that expert demonstrations arise from some MFE. Finally, we note that for finite samples, estimated quantities $\hat{\mu}_E$ and $\langle \hat{\Phi} \rangle_{\pi_E,\mu_E}$ incur estimation error. The $L$-smoothness established in Proposition~5.1 ensures the objective varies smoothly under such perturbations, providing robustness. Formal finite-sample analysis, incorporating concentration bounds, is an important direction for future work.
\end{remark}}





\section{Maximum Causal Entropy IRL Problem}

This section formulates the maximum causal entropy inverse reinforcement learning problem and presents an equivalent optimization problem, which provides the foundation for applying Lagrangian relaxation in the subsequent analysis.


In the standard IRL framework, we assume the expert's behaviour is generated by a MFE $(\pi_E,\mu_E)$ under some unknown reward function $r_E$ in the RKHS $\cH$, and the equilibrium is characterized by two conditions: (i) the expert's policy $\pi_E$ is optimal for the mean-field $\mu_E$ under the reward $r_E$, and (ii) $\mu_E$ is the stationary distribution induced by $\pi_E$ with the mean-field term fixed at $\mu_E$. 

However, the inverse problem is inherently ill-posed as multiple distinct reward functions within $\cH$ may explain identical expert demonstrations. To resolve this ambiguity, we adopt the maximum causal entropy principle, which regularizes the problem by selecting the policy with the highest causal entropy from the set of all policies that are consistent with the expert's feature expectations, thereby introducing a minimally informative prior bias. The discounted causal entropy of a policy $\pi$ is defined\footnote{\tb{ The optimal policy derived via the soft Bellman optimality equations (Section IV) yields a softmax distribution that is strictly positive for all actions, ensuring the entropy is well-defined and finite.}
} as
$$
H(\pi) = \sum_{t=0}^{\infty} \beta^t E^{\pi,\mu_E} \left[-\log \, \pi(a(t)|x(t)) \right].
$$

Building on this, we formulate the kernel-based maximum discounted causal entropy IRL problem as the following constrained optimization:
\begin{align*}
&\mathbf{(OPT_1)} \ \text{maximize}_{\pi \in \mathcal{P}(\mathcal{A}|\sX)} \ H(\pi) ~~ \text{subject to:} \\
&\sum_{(y,a) \in \sX \times \mathcal{A}} p(x|y,a,\mu_E) \pi(a|y) \mu_E(y) = \mu_E(x) ~~ \forall x \in \sX, \\
&\hspace{1.25em}\sum_{t=0}^{\infty} \beta^t E^{\pi,\mu_E}[\Phi(x(t),a(t),\mu_E)] = \langle \Phi \rangle_{\pi_E,\mu_E},
\end{align*}
where, $\Pnew(\sA|\sX)$ is the set of stochastic kernels from $\sX$ to $\sA$, and the expectation in the last constraint is taken in the Bochner integral sense. 

The optimal solution of \(\mathbf{(OPT_1)}\), together with the mean-field term \(\mu_E\), constitutes an equilibrium.

\begin{proposition} \label{pr2}
Let $\pi^*$ be the solution of $\mathbf{(OPT_1)}$. Then, the pair $(\pi^*,\mu_E)$ is a mean-field equilibrium.
\end{proposition}

\begin{proof}
To establish that $(\pi^*, \mu_E)$ constitutes a MFE, we must verify that $\mu_E \in \Lambda(\pi^*)$ and $\pi^*$ is optimal with respect to $\mu_E$ (i.e., $\pi^* \in \Psi(\mu_E)$ ).
The first constraint in (OPT1),
\begin{equation*}
\mu_E(x) = \sum_{(a,y) \in \A \times \sX} p(x | y, a, \mu_E) \pi^*(a|y) \mu_E(y) ~~ \forall x \in \sX,
\end{equation*}
ensures that $\mu_E$ is invariant under the dynamics induced by the policy $\pi^*$ \lk{when the mean-field term is fixed as $\mu_E$.} This implies $\mu_E \in \Lambda(\pi^*)$ by definition.

Let $r_E \in \mathcal{H}$ denote the true, unknown expert reward function corresponding to the \lk{MFE} $(\pi_E, \mu_E)$. The equilibrium conditions $\mu_E \in \Lambda(\pi_E)$ and $\pi_E \in \Psi(\mu_E)$ implies the expert policy is optimal for its own mean-field. That is,
\[
J_{\mu_E}(\pi_E, \mu_E) = \sup_{\pi \in \Pi} J_{\mu_E}(\pi, \mu_E).
\]
Furthermore, by the reproducing property of the Hilbert space $\cH$ and the definition of the feature expectation $\langle \Phi \rangle_{\pi_E,\mu_E}$, the expert's value can be expressed as the inner product
$$J_{\mu_E}(\pi_E, \mu_E) = \langle r_E, \langle \Phi \rangle_{\pi_E,\mu_E} \rangle_{\cH}.$$ 
The second constraint in (OPT1) enforces matching feature expectations
$$
\sum_{t=0}^{\infty} \beta^t \, E^{\pi^*,\mu_E}[\Phi(x(t),a(t),\mu_E)] = \langle \Phi \rangle_{\pi_E,\mu_E}.
$$
This directly implies $J_{\mu_E}(\pi^*, \mu_E) = \langle r_E, \langle \Phi \rangle_{\pi_E,\mu_E} \rangle_{\cH}$. By the definition of the expert equilibrium, the right-hand side equals the supremum $\sup_{\pi \in \Pi} J_{\mu_E}(\pi, \mu_E)$. Hence,
\[
J_{\mu_E}(\pi^*, \mu_E) = \sup_{\pi \in \Pi} J_{\mu_E}(\pi, \mu_E),
\]
proving that $\pi^* \in \Psi(\mu_E)$.
\end{proof}

\tb{ To align the formulation with the discounted infinite-horizon setting, we replace the stationarity constraint on the state distribution with an equivalent constraint on discounted state-occupancy frequencies. Specifically, replacing
}
$$\mu_E(x) = \sum_{(a,y) \in \sA \times \sX} p(x|y,a,\mu_E) \, \pi(a|y) \, \mu_E(y) \,\, \forall x \in \sX$$ 
in \(\mathbf{(OPT_1)}\) with
$$
\sum_{t=0}^{\infty} \beta^t \, E^{\pi,\mu_E} [1_{\{x(t)=x\}}] = \mu_E(x)/(1-\beta)  \,\, \forall x \in \sX,
$$
\tb{ leads to the following alternative formulation. In this form, the discounted occupancy constraint is naturally incorporated via Lagrangian relaxation and admits an entropy-regularized MDP interpretation.
}

\renewcommand\arraystretch{1.5}
\[
\begin{array}{ll}
\mathbf{(\widehat{OPT_1})} 
& \text{maximize}_{\pi \in \Pnew(\sA|\sX)} \ H(\pi)  ~~  \text{subject to:} \\[2pt]
& \displaystyle \sum_{t=0}^{\infty} \beta^t \, E^{\pi,\mu_E} \big[ \mathbf{1}_{\{x(t)=x\}} \big] 
    = \frac{\mu_E(x)}{1-\beta}, \quad \forall x \in \sX, \\[6pt]
& \displaystyle \sum_{t=0}^{\infty} \beta^t \, E^{\pi,\mu_E} \big[ \Phi(x(t),a(t),\mu_E) \big] 
    = \langle \Phi \rangle_{\pi_E,\mu_E}.
\end{array}
\]


\tb{The following result formalizes the relationship between the two formulations.}

\begin{proposition} \label{pr3}
	$\mathbf{(\widehat{OPT_1})}$ and $\mathbf{(OPT_1)}$ are equivalent. 
\end{proposition}

\begin{proof}
	Suppose $\mu_E$ is stationary for the policy $\pi$ and mean-field $\mu_E$, meaning
	$$
	\mu_E(x) = \sum_{(a,y) \in \sA \times \sX} p(x|y,a,\mu_E) \, \pi(a|y) \, \mu_E(y) ~~ \forall x \in \sX.
	$$
	Then, the state process ${x(t)}$ is stationary with $\Law{x(t)} = \mu_E$ for all $t \geq 0$. Consequently, the discounted state occupancy is
	$$
	\sum_{t=0}^{\infty} \beta^t \, E^{\pi,\mu_E} [1_{\{x(t)=x\}}] = \sum_{t=0}^{\infty} \beta^t \, \mu_E(x) = \mu_E(x)/(1-\beta)
    $$
	for every $x \in \sX$. Conversely, suppose the discounted occupancy matches the stationary distribution
	$$
	\sum_{t=0}^{\infty} \beta^t \, E^{\pi,\mu_E} [1_{\{x(t)=x\}}] =   \mu_E(x)/(1-\beta)  ~~ \forall x \in \sX.
	$$
    Let $\nu_{\pi}^{\mathcal{X}}$ be the normalized state occupation measure, defined by
    $$
	\nu_{\pi}(x,a) \coloneqq (1-\beta) \, \sum_{t=0}^{\infty} \beta^t \, E^{\pi,\mu_E} \left[ 1_{\{(x(t),a(t)) = (x,a)\}} \right].
	$$
    By assumption $\nu_{\pi}^{\sX}(x)=\mu_E(x)$ for all $x \in \sX$. Furthermore, $\nu_{\pi}^{\sX}$ satisfies the Bellman flow equation
$$
	\nu_{\pi}^{\sX}(x) = (1-\beta) \, \mu_E(x)  + \beta \, \sum_{(y,a) \in \sX\times\sA} p(x|y,a,\mu_E) \, \pi(a|y) \, \nu_{\pi}^{\sX}(y).$$
	
    Substituting $\nu_{\pi}^{\sX}(x)=\mu_E(x)$ into this equation and solving for $\mu_E(x)$ simplifies to the stationarity condition
    $$
	\mu_E(x) = \sum_{(a,y) \in \sA \times \sX} p(x|y,a,\mu_E) \, \pi(a|y) \, \mu_E(y) \,\, \forall x \in \sX. 
	$$
    This completes the proof.
\end{proof}

\tb{
\begin{remark}
We emphasize that the discounted \tb{reward} structure in our formulation  pertains to the weighting of future rewards, not to transient state 
dynamics. Under the stationarity assumption $\mu_0 = \mu_E$, the state  distribution is time-invariant, and the discount factor $\beta$ governs  the relative importance of near-term versus long-term rewards. The time-varying formulation with transient dynamics is treated  separately in Section~VI.
\end{remark}
}

\section{Lagrangian Relaxation and the Log-likelihood Formulation}

 

This section analyzes the Lagrangian relaxation of the optimization problem \(\mathbf{(OPT_1)}\) using its equivalent formulation $\mathbf{(\widehat{OPT_1})}$. This approach transforms the constrained IRL problem into a maximum likelihood estimation framework.


\tb{While Lagrangian relaxation for IRL has been studied in finite-horizon MDPs~\cite{GlTo22,ZiBaDe10,ZiBaDe13}, extending to infinite-horizon stationary MFGs introduces fundamental technical challenges. First, trajectory distributions become ill-defined on infinite path spaces, necessitating the causal entropy principle rather than classical maximum entropy. Second, unlike finite-horizon backward induction, the soft Bellman equation defines $Q^{\theta}$ as a fixed point of a contraction operator; differentiating through this fixed point requires proving Fr\'{e}chet differentiability via the implicit function theorem (Theorem~4.1). Third, establishing $L$-smoothness requires bounding infinite series through contraction arguments (Proposition~5.1), whereas finite-horizon objectives involve only finite sums. Finally, beyond standard MDPs, our MFG formulation couples policy optimization with population dynamics through the stationarity constraint, requiring the equivalent reformulation in Proposition~3.2. }

We introduce the Lagrange multiplier $\theta \triangleq (\lambda,h) \in \R^{\sX} \times \cH$ and define the Lagrangian relaxation associated with $\mathbf{(\widehat{OPT_1})}$ as
\begin{align*}
  \cG(\theta)  &\triangleq \max_{\pi \in \Pnew(\sA|\sX)} \cL(\pi,\theta) \\ 
  &\triangleq \max_{\pi \in \Pnew(\sA|\sX)}   H(\pi) + \left \langle \lambda, \sum_{t=0}^{\infty} \beta^t \, E^{\pi,\mu_E}[1_{\{x(t)=\cdot\}}] - \mu_E(\cdot) \right \rangle_{\R^{\sX}} 
 \\ & +\left \langle h, \sum_{t=0}^{\infty} \beta^t \, E^{\pi,\mu_E}[\Phi(x(t),a(t),\mu_E)] - \langle \Phi \rangle_{\pi_E,\mu_E} \right \rangle_{\cH}.
\end{align*}
Here, $\langle \cdot,\cdot \rangle_{\R^{\sX}}$ denotes the standard inner product in the Euclidean space $\R^{\sX}$ and $\langle \cdot,\cdot \rangle_{\cH}$ represents the inner product in the RKHS $\cH$. Then,
$$
\mathbf{(\widehat{OPT_1})} \leq \min_{\theta} \cG(\theta) \triangleq \cG(\theta^*).
$$
Since the terms $\langle \lambda, \mu_E \rangle_{\R^{\sX}}$ and $\langle h, \langle \Phi \rangle_{\pi_E,\mu_E} \rangle_{\cH}$ are independent of the policy $\pi$, \tb{they do not affect the maximization over $\pi$ and can be omitted from the inner optimization.} 
This leads to the simplified problem:

$$\max_{\pi \in \Pnew(\sA|\sX)} \text{ }  H(\pi) + \sum_{t=0}^{\infty} \beta^t \, E^{\pi,\mu_E}[\lambda(x(t))] + \sum_{t=0}^{\infty} \beta^t \, E^{\pi,\mu_E}[h(x(t),a(t),\mu_E)],$$
where 
\begin{align*}
\left \langle \lambda, \sum_{t=0}^{\infty} \beta^t \, E^{\pi,\mu_E}[1_{\{x(t)=\cdot\}}]\right \rangle_{\R^{\sX}} &= \sum_{t=0}^{\infty} \beta^t \, E^{\pi,\mu_E}[\lambda(x(t))]
\end{align*}
and
$$\left \langle h, \sum_{t=0}^{\infty} \beta^t \, E^{\pi,\mu_E}[\Phi(x(t),a(t),\mu_E)] \right \rangle_{\cH}  = \sum_{t=0}^{\infty} \beta^t \, E^{\pi,\mu_E}[h(x(t),a(t),\mu_E)].$$

In the last equality, we use the reproducing property of the feature function $\Phi$, that is, $\langle h,\Phi(x,a,\mu) \rangle_{\cH} = h(x,a,\mu)$. Note that the above problem is 
indeed an entropy regularized MDP with the reward function 
$$r_{\theta}(x,a) \triangleq \lambda(x) + h(x,a,\mu_E).$$ 
The solution to this problem is given by the following soft Bellman optimality equations:
\begin{align*}
Q^{\theta}(x,a) &= r_{\theta}(x,a) + \beta \, \sum_{y \in \sX} V^{\theta}(y) \, p(y|x,a,\mu_E) \\
V^{\theta}(x) &= \log \sum_{a \in \sA} e^{Q^{\theta}(x,a)} \triangleq \softmax_{a \in \sA} Q^{\theta}(x,a).
\end{align*}
Then, it follows that
$$
\pi^{\theta}(a|x) = e^{Q^{\theta}(x,a)-V^{\theta}(x)}
$$
is the optimal solution (see \cite{NeJoGo17}). Here, due to the additional entropy reward, we simply replace the $\max$-operator with $\softmax$-operator in the classical Bellman recursion. 

%

Directly solving the constraints of \( \mathbf{\widehat{(OPT_1)}} \) for the optimal parameter 
\( \theta^* \) is often computationally intractable.  Instead, we derive an alternative objective (surrogate) function whose stationary points yield \( \theta^* \), and we find these points via gradient ascent. This requires the Fr\'{e}chet differentiability of $Q^{\theta}$ and $V^{\theta}$ in $\theta$. Let  $\ee_x \in \R^{\sX}$ denote the standard basis vector defined by $\ee_x(y) = 1_{\{x=y\}}$. We then construct the vector-valued feature function
$$
f(x,a) \triangleq \begin{bmatrix}
\ee_x \\ \Phi(x,a,\mu_E) 
\end{bmatrix} \in \R^{\sX} \times \cH \triangleq \cal W.
$$


\begin{theorem}
Let $\cal W = \mathbb{R}^{\sX} \times \mathcal{H}$ be a Hilbert space endowed with the inner product
\[
\langle \theta_1, \theta_2 \rangle_{\cal W} \triangleq 
\langle \lambda_1, \lambda_2 \rangle_{\mathbb{R}^{\sX}} + \langle h_1, h_2 \rangle_{\mathcal{H}},
\]
where $\theta = (\lambda,h) \in \cal W$. Let $Q^{\theta}: \sX \times \mathcal{A} \to \mathbb{R}$ be the unique fixed point of the $\beta$-contraction operator
\[
T^{\theta} Q(x,a) \triangleq \langle \theta, f(x,a) \rangle_{\cal W} + \beta \sum_{y \in \sX} V(y) \, p(y|x,a,\mu_E),
\]
with
\[
V(y) \triangleq \softmax_{a \in \mathcal{A}} Q^{\theta}(y,a).
\]
Then, $Q^{\theta}$ and $V^{\theta}: \sX \to \mathbb{R}$ are Fr\'echet differentiable with respect to $\theta \in \cal W$.
\end{theorem}

\begin{proof}
Define the mapping
\[
F(Q,\theta) \triangleq Q - T^{\theta} Q,
\]
for $Q \in \mathbb{R}^{\sX \times \mathcal{A}}$ and $\theta \in \cal W$. By construction, $F(Q^{\theta},\theta) = 0$, and $Q^{\theta}$ is the unique solution of $F(Q,\theta) = 0$ due to the $\beta$-contraction property of $T^{\theta}$ under the sup-norm. 
Since $F(Q,\theta)$ is linear in $\theta$ and the softmax function
\[
\mathbb{R}^{\mathcal{A}} \ni l \mapsto \softmax_{a \in \mathcal{A}} l(a) \in \mathbb{R}
\]
is continuously differentiable, $F(Q,\theta)$ is continuously Fr\'echet differentiable with respect to $(Q,\theta) \in \mathbb{R}^{\sX \times \mathcal{A}} \times \cal W$. Moreover, the Jacobian of $F(\cdot,\theta): \mathbb{R}^{\sX \times \mathcal{A}} \to \mathbb{R}^{\sX \times \mathcal{A}}$ with respect to $Q$ is given by
\[
\nabla_Q F(Q,\theta) = I - \beta D_Q,
\]
where
\[
D_Q(x,a \mid y,b) = \pi^Q(b|y) \, p(y|x,a,\mu_E),\] and
\[\pi^Q(b|y) \triangleq \frac{e^{Q(y,b)}}{\sum_{a \in \mathcal{A}} e^{Q(y,a)}}.
\]
Since $D_Q$ is a transition matrix, $I - \beta D_Q$ is invertible for $\beta < 1$. By the implicit function theorem, the mapping $\cal W \ni \theta \mapsto Q^{\theta} \in \mathbb{R}^{\sX \times \mathcal{A}}$ is Fr\'echet differentiable.

Finally, the Fr\'echet differentiability of $V^{\theta}$ follows from that of $Q^{\theta}$, the continuous differentiability of the softmax function, and the chain rule, since $V^{\theta}(y) = \softmax_{a \in \mathcal{A}} Q^{\theta}(y,a)$.
\end{proof}

For any policy $\pi$, we define the un-normalized state-action occupation measure as 
$$
\gamma_{\pi}(x,a) \coloneqq \sum_{t=0}^{\infty} \beta^t \, E^{\pi,\mu_E} \left[ 1_{\{(x(t),a(t)) = (x,a)\}} \right].
$$
We consider the following objective function, whose stationary points will correspond to solutions of the IRL problem:
$$
\V(\theta) \triangleq \sum_{(x,a) \in (\sX\times\sA)} \log \pi^{\theta}(a|x) \, \gamma_{\pi_E}(x,a).
$$
The following theorem establishes the connection between the stationary points of $\V$ and the solutions of the original optimization problem.

\begin{theorem}\label{equivv3}
If $\theta^*$ satisfies $\nabla \V(\theta^*) = 0$, then
\begin{enumerate}[(i)]
\item $\theta^*$ minimizes the dual function: $\theta^* \in \argmin_{\theta \in \mathcal{W}} \cG(\theta).$
\item The induced policy $\pi^{\theta^*}$ is optimal for the primal problem: $\pi^{\theta^*} \in \argmax  \mathbf{\widehat{(OPT_1)}}$.
\end{enumerate}
\end{theorem}

\begin{proof}
We begin by expressing $\V(\theta)$ in terms of the advantage function as 
$$
\V(\theta) = E^{\pi_E,\mu_E} \left[\sum_{t=0}^{\infty} \beta^t \, \left(Q^{\theta}(x(t),a(t))-V^{\theta}(x(t))\right) \right].
$$
Note that
\begin{align}
\nabla r_{\theta}(x,a) = f(x,a). \label{neq1} 
\end{align}
Moreover, the gradient of the value function can be derived as 
\begin{align}
\nabla V^{\theta}(x) &= \nabla \log \sum_{a \in \sA} e^{Q^{\theta}(x,a)} \nonumber \\
&= \frac{1}{\sum_{a \in \sA} e^{Q^{\theta}(x,a)}} \sum_{a \in \sA} e^{Q^{\theta}(x,a)} \, \nabla Q^{\theta}(x,a) \nonumber \\
&= \sum_{a \in \sA} \nabla Q^{\theta}(x,a) \, \pi^{\theta}(a|x) \label{neq2}
\end{align}
where the policy is given by the softmax distribution
\begin{align*}
\pi^{\theta}(a|x) = \frac{e^{Q^{\theta}(x,a)}}{e^{V^{\theta}(x)}} = \frac{e^{Q^{\theta}(x,a)}}{\sum_{a \in \sA} e^{Q^{\theta}(x,a)}}.
\end{align*}
Furthermore, the Bellman equation implies that the gradient of the action-value function satisfies 
\begin{align}
\nabla Q^{\theta}(x,a) &= f(x,a) + \beta \, \sum_{y \in \sX} \nabla V^{\theta}(y) \, p(y|x,a,\mu_E). \label{neq3}
\end{align}
By recursively applying \eqref{neq2} within \eqref{neq3}, we obtain the following closed-form solution:

$$\nabla Q^{\theta}(x(0),a(0)) =  E^{\pi^{\theta}} \left[ \sum_{t=0}^{\infty} \beta^t f(x(t),a(t)) \bigg| x(0),a(0)\right]. \label{neq4}$$
This result follows by successively unrolling the recursion over an infinite horizon. Indeed,
\begin{align*}
\nabla Q^{\theta}(x(0),a(0))  &= f(x(0),a(0)) + \beta  \sum_{x(1) \in \sX} \nabla V^{\theta}(x(1)) \, p(x(1)|x(0),a(0),\mu_E)\\
&= f(x(0),a(0))  + \beta \! \sum_{x(1) \in \sX} \sum_{a(1) \in \sA} \nabla Q^{\theta}(x(1),a(1)) \, \pi_{\theta}(a(1)|x(1)) p(x(1)|x(0),a(0),\mu_E) \\
&= f(x(0),a(0)) + \beta \! \sum_{x(1) \in \sX} \sum_{a(1) \in \sA} \bigg[f(x(1),a(1)) \\ &\phantom{x} +\beta \sum_{x(2) \in \sX} \nabla V^{\theta}(x(2)) \, p(x(2)|x(1),a(1),\mu_E)\bigg] \pi_{\theta}(a(1)|x(1))\, p(x(1)|x(0),a(0),\mu_E) \\
&\phantom{x}\vdots \\
&= E^{\pi^{\theta}} \left[ \sum_{t=0}^{N-1} \beta^t f(x(t),a(t)) \bigg| x(0),a(0)\right] + \beta^N \, E^{\pi^{\theta}} \left[ \nabla V^{\theta}(x(N)) \bigg| x(0),a(0)\right] \\
&\rightarrow E^{\pi^{\theta}} \left[ \sum_{t=0}^{\infty} \beta^t f(x(t),a(t)) \bigg| x(0),a(0)\right] \,\, \text{as} \,\, N \rightarrow \infty.
\end{align*} 
Combining the results derived above, it follows that
\begin{align*}
&\nabla \V(\theta) = E^{\pi_E,\mu_E} \left[\sum_{t=0}^{\infty} \beta^t \, \left(\nabla Q^{\theta}(x(t),a(t))-\nabla V^{\theta}(x(t))\right) \right] \\
&\stackrel{(\text{by} \,\, (\ref{neq3}))}{=} E^{\pi_E,\mu_E} \bigg[\sum_{t=0}^{\infty} \beta^t \, \Big(f(x(t),a(t)) + \beta \! \! \sum_{y(t+1) \in \sX} \!\!\!\!\! \nabla V^{\theta}(y(t+1)) p(y(t+1)|x(t),a(t),\mu_E) -\nabla V^{\theta}(x(t)) \Big) \bigg] \\
&= \langle f \rangle_{\pi_E,\mu_E} + E^{\pi_E,\mu_E} \left[\sum_{t=1}^{\infty} \beta^t \, \nabla V^{\theta}(x(t)) \right] - E^{\pi_E,\mu_E} \left[\sum_{t=0}^{\infty} \beta^t \, \nabla V^{\theta}(x(t)) \right] \\
&= \langle f \rangle_{\pi_E,\mu_E}  - \sum_{x(0) \in \sX}  \nabla V^{\theta}(x(0)) \, \mu_E(x(0)) \\ &\stackrel{(\text{by} \,\, (\ref{neq2}))}{=} \langle f \rangle_{\pi_E,\mu_E}  - \sum_{(x(0),a(0)) \in \sX \times \sA}  \nabla Q^{\theta}(x(0),a(0)) \pi^{\theta}(a(0)|x(0)) \, \mu_E(x(0)) \\
&\stackrel{(\text{by} \,\, (\ref{neq4}))}{=} \langle f \rangle_{\pi_E,\mu_E}  - \sum_{(x(0),a(0)) \in \sX \times \sA}   E^{\pi^{\theta}} \left[ \sum_{t=0}^{\infty} \beta^t \, f(x(t),a(t)) \bigg| x(0),a(0)  \right] \pi^{\theta}(a(0)|x(0)) \, \mu_E(x(0)) \\
&= \langle f \rangle_{\pi_E,\mu_E}  -   E^{\pi^{\theta},\mu_E} \left[ \sum_{t=0}^{\infty} \beta^t f(x(t),a(t)) \right].
\end{align*}
Therefore, $\nabla \V(\theta^*) = 0$ if and only if 
$$
\langle f \rangle_{\pi_E,\mu_E}  =   E^{\pi^{\theta^*},\mu_E} \left[ \sum_{t=0}^{\infty} \beta^t f(x(t),a(t)) \right].
$$
On the other hand,
$$E^{\pi^{\theta},\mu_E} \left[ \sum_{t=0}^{\infty} \beta^t f(x(t),a(t)) \right] \\ = \begin{bmatrix}
E^{\pi^{\theta},\mu_E} \left[ \sum_{t=0}^{\infty} \beta^t 1_{\{x(t) = \cdot\}} \right] \\ E^{\pi^{\theta},\mu_E} \left[ \sum_{t=0}^{\infty} \beta^t \Phi(x(t),a(t),\mu_E) \right]
\end{bmatrix} \in \R^{\sX} \times \cH$$
and
\begin{align*}
\langle f \rangle_{\pi_E,\mu_E} &= \begin{bmatrix}
\mu_E \\ \langle \Phi \rangle_{\pi_E,\mu_E}
\end{bmatrix} \in \R^{\sX} \times \cH.
\end{align*}
Hence, $\nabla \V(\theta^*) = 0$ if and only if $\pi^{\theta^*}$ satisfies the constraints in $\mathbf{(\widehat{OPT_1})}$. 
Note that we have
\lk{\begin{align*}
\mathbf{(\widehat{OPT_1})} &= \max_{\pi} \min_{\theta} \cL(\pi,\theta) 
\leq \min_{\theta} \max_{\pi} \cL(\pi,\theta) = \min_{\theta} \cG(\theta) \\
&\leq \max_{\pi} \cL(\pi,\theta^*) = \cG(\theta^*) = \cL(\pi^{\theta^*}, \theta^*) \\
&\leq \mathbf{(\widehat{OPT_1})} \quad \text{(since $\pi^{\theta^*}$ is feasible for $\mathbf{(\widehat{OPT_1})}$).}
\end{align*}}
Therefore, if $\nabla \V(\theta^*) = 0$, then
\begin{align*}
\theta^* \in \argmin_{\theta} \cG(\theta), \,\, \pi^{\theta^*} \in \argmax  \mathbf{\widehat{(OPT_1)}}.
\end{align*}
This completes the proof.

\end{proof}

\section{\lk{Maximum Log-Likelihood Gradient Ascent Algorithm}}

We now propose a gradient ascent algorithm to locate a stationary point of the objective function $\mathcal{V}(\theta)$ and analyze its convergence properties. According to Theorem~\ref{equivv3}, solving $\mathbf{\widehat{(OPT_1)}}$ is equivalent to the following root-finding problem:
\[
\mathbf{(MLL)} \quad \text{Find } \theta^* \text{ such that } \nabla \mathcal{V}(\theta^*) = 0.
\]
This problem can be conceptualized as an instance of maximum log-likelihood estimation. To be able to apply a constant step-size gradient ascent algorithm for finding the stationary point of the function $\V(\theta)$, we need to establish that $\V$ is $L$-smooth for some $L>0$. Before addressing this, we first examine the structure of the gradient of $\V$.

Theorem~\ref{equivv3} establishes that the gradient of the objective function is given by
$$
\nabla \V(\theta) = \langle f \rangle_{\pi_E,\mu_E} - E^{\pi^{\theta},\mu_E} \left[ \sum_{t=0}^{\infty} \beta^t f(x(t),a(t)) \right].
$$
The expert's feature expectation vector $\langle f \rangle_{\pi_E,\mu_E}$ is known \textit{a priori} from the demonstration data and is defined as
$$
\langle f \rangle_{\pi_E,\mu_E} = \begin{bmatrix}
\mu_E \\ \langle \Phi \rangle_{\pi_E,\mu_E}
\end{bmatrix} \in \R^{\sX} \times \cH.
$$
Consequently, the computation of the gradient \( \nabla \V(\theta) \) reduces to evaluating the expected discounted sum of $f$ under the policy \( \pi^{\theta} \) and the mean-field distribution \( \mu_E \). 

This can be achieved using the unnormalized state-action occupation measure $\gamma_{\pi_\theta}$, as
\begin{align*}
E^{\pi^{\theta},\mu_E} \left[ \sum_{t=0}^{\infty} \beta^t f(x(t),a(t)) \right] 
= \sum_{(x,a) \in \sX\times\sA} f(x,a) \, \gamma_{\pi^{\theta}}(x,a).
\end{align*}
To compute this in practice, first obtain the soft value function $V^\theta$ as the unique fixed point of the $\beta$-contraction operator
\small
$$L^{\theta} V(x) \triangleq 
\max_{\pi \in \Pnew(\sA)} 
\sum_{a \in \sA} \bigg\{
r_{\theta}(x,a,\mu_E)  \\ 
+ \beta \sum_{y \in \sX} V(y) \, p(y|x,a,\mu_E) \bigg\} \pi(a) + H(\pi) .$$

\normalsize
Note that the application of $L^\theta$ is straightforward when the entropy term $H(\pi)$ is present; in its absence, solving for $L^\theta V$ can be computationally intractable.

By the variational formula, $V^\theta$ can be computed via value iteration: start with $V_0$ and iteratively update $V_{t+1} = L^\theta V_t$ for $t = 0, 1, \dots$, yielding $V_t \to V^\theta$. This maximization is explicit due to the entropy term
$$L^{\theta} V(x) =
\log \sum_{a \in \sA} 
e^{\,r_{\theta}(x,a,\mu_E) + \beta \sum_{y \in \sX} V(y) \, p(y|x,a,\mu_E)}.$$

Once $V^\theta$ is computed, we obtain
$$Q^{\theta}(x,a) 
= r_{\theta}(x,a,\mu_E)  
+ \beta \sum_{y \in \sX} V^{\theta}(y) \, p(y|x,a,\mu_E),$$
and the corresponding policy
\[
\pi^{\theta}(a|x) = e^{Q^{\theta}(x,a)-V^{\theta}(x)}.
\]
By the Bellman flow condition, the state-occupation measure $\gamma^X_{\pi_\theta}$ satisfies
\begin{align*}
\gamma_{\pi^{\theta}}^{\sX}(\cdot) 
= \mu_E(\cdot) 
+ \beta \sum_{(x,a) \in \sX\times\sA} p(\cdot|x,a,\mu_E) \, \pi^{\theta}(a|x) \, \gamma_{\pi^{\theta}}^{\sX}(x).
\end{align*}
This relationship is equivalent to the linear system
\[
\gamma_{\pi^{\theta}}^{\sX} = \mu_E + A^{\theta} \, \gamma_{\pi^{\theta}}^{\sX},
\]
where
\[
A^{\theta}(x,y) \triangleq \sum_{a \in \sA} p(y|x,a,\mu_E) \, \pi^{\theta}(a|x).
\]
Given that $A^\theta$ is a valid transition matrix, the matrix $(I - \beta A^{\theta})$ is invertible. Consequently, the state occupation measure is given explicitly by
\[
\gamma_{\pi^{\theta}}^{\sX} = (I - \beta \,A^{\theta})^{-1} \, \mu_E.
\]
The full state-action occupation measure is then obtained as
\[
\gamma_{\pi^{\theta}} = \gamma_{\pi^{\theta}}^{\sX} \otimes \pi^{\theta}.
\]
 Equipped with this closed-form expression for the occupation measure, we now introduce a maximum log-likelihood gradient ascent algorithm to locate a stationary point of $\mathcal{V}(\theta)$. 
 

\begin{algorithm}[H]
	\caption{\lk{Maximum Log-Likelihood Gradient Ascent}}
	\label{GD}
	\begin{algorithmic}[1]
\STATE \textbf{Input:} initial parameter $\btheta_0$, stepsize $\gamma > 0$, iterations $K$
\STATE Initialize $\btheta \leftarrow \btheta_0$
\FOR{$k = 0,\ldots,K-1$}
    \STATE $\btheta_{k+1} \leftarrow \btheta_k + \gamma \nabla \mathcal{V}(\btheta_k)$
\ENDFOR
\STATE \textbf{Return:} $\btheta_K$ and $\pi^{\btheta_K}$
\end{algorithmic}

\end{algorithm}


Once the algorithm recovers the optimal parameters \(\theta^\star\) and the corresponding policy \(\pi_{\theta^\star}\), each agent executes \(\pi_{\theta^\star}(\cdot| x)\) using only its local state information $x$ and the mean-field distribution $\mu_E$, which is either directly observable from aggregate population statistics or estimated from empirical state frequencies of the population without requiring coordination among agents. Thus, although the inverse learning problem is inherently centralized due to its reliance on aggregated population data, the forward problem of computing equilibrium responses and executing the policy is fully decentralized.


\tb{Once the algorithm recovers the optimal parameters $\theta^\star$ 
and the corresponding policy $\pi_{\theta^\star}$, each agent executes 
$\pi_{\theta^\star}(\cdot \mid x)$ using only its local state information $x$ 
without requiring coordination among agents. Thus, although the IRL is inherently centralized due  to its reliance on aggregated population data, the forward problem of computing 
equilibrium responses and executing the policy is fully decentralized.}

The following result proves the $L$-smoothness of $\V$ for some $L>0$, a sufficient condition for the convergence of gradient ascent (Algorithm~\ref{GD}) to a stationary point of $\V$.

\tb{
\begin{proposition}\label{prop}
The function $\V$ is $L$-smooth, where 
\begin{align*}
L \triangleq \frac{K^2}{(1-\beta)^2} \, \left(\frac{2 \, \beta}{1-\beta} + 1 \right), \quad
K \triangleq \max_{(x,a) \in \sX \times \sA} \| f(x,a) \|_{\R^{\sX} \times \cH}.
\end{align*} 
\end{proposition}
}

\begin{proof}
Our primary objective is to establish the Lipschitz continuity of the gradient $\nabla Q^{\theta}(x, a)$ for all state-action pairs $(x, a) \in \sX \times \sA$; that is, to prove that the function $\theta \mapsto Q^{\theta}(x, a)$ is smooth. 
To this end, we introduce the following operators (or matrices, in the case of finite-dimensional input and output spaces):
\begin{align*}
&\nabla V^{\theta} \in (\R^{\sX}) \times (\R^{\sX} \times \cH), \,\, \nabla Q^{\theta} \in (\R^{\sX \times \sA}) \times (\R^{\sX} \times \cH), \\ 
&F \in (\R^{\sX \times \sA}) \times (\R^{\sX} \times \cH): F(x,a) \triangleq f(x,a) \in \R^{\sX} \times \cH, \\
&P \in (\R^{\sX \times \sA}) \times (\R^{\sX}): [[P]]_{(x,a),y} \triangleq p(y|x,a,\mu_E), \\ 
&\Pi(\theta) \in (\R^{\sX}) \times (\R^{\sX \times \sA}): [[\Pi(\theta)]]_{y,(x,a)} \triangleq 1_{\{x=y\}} \, \pi^{\theta}(a|x).
\end{align*}
Derived within the proof of Theorem~\ref{equivv3}, the gradients of the value function $V^{\theta}$ and the state-action value function $Q^{\theta}$ satisfy the system of coupled equations
\begin{align*}
\nabla V^{\theta} = \Pi(\theta) \, \nabla Q^{\theta} \quad \text{and} \quad 
\nabla Q^{\theta} = F + \beta \, P \, \nabla V^{\theta}.
\end{align*}
Substituting the first identity into the second yields the closed-form expression for the gradient
\begin{align*}
\nabla Q^{\theta} = (I -\beta \, P \, \Pi(\theta))^{-1} \, F.
\end{align*}
This is well-defined because $P \Pi(\theta)$ constitutes a transition matrix on the space $\R^{\sX \times \sA} \times \R^{\sX \times \sA}$, ensuring that the inverse is given by the convergent Neumann series
$$
(I -\beta \, P \, \Pi(\theta))^{-1} = \sum_{k=0}^{\infty} \beta^k \, (P \, \Pi(\theta))^{k}.
$$
The matrix $\Pi(\theta)$ is defined through the policy $\pi^{\theta}$; therefore, we begin by examining how $\pi^{\theta}$ depends on $\theta$. \tb{ Note that the $l_{\infty} \to l_1$-norm of Jacobian of softmax is less than $1$. This follows from the fact that each row of the Jacobian of the softmax function sums to zero, and the entries are bounded by the softmax probabilities.
Therefore, by the mean-value theorem, for any $x \in \sX$, we have 
$$
\|\pi^{\theta_1}(\cdot|x) - \pi^{\theta_2}(\cdot|x)\|_1 \leq \|Q^{\theta_1}(x,\cdot) - Q^{\theta_2}(x,\cdot)\|_{\infty}.
$$
}
Therefore, we have 
\begin{align*}
\|\Pi(\theta_1)-\Pi(\theta_2)\|_{\infty} &\triangleq \max_{y \in \sX} \|\Pi(\theta_1)(y|\cdot) - \Pi(\theta_2)(y|\cdot)\|_1 \\ &= \max_{y \in \sX} \|\pi^{\theta_1}(\cdot|y) - \pi^{\theta_2}(\cdot|y)\|_1 \\
&\leq \max_{y \in \sX} \|Q^{\theta_1}(y,\cdot) - Q^{\theta_2}(y,\cdot)\|_{\infty}, 
\end{align*}
where $\Pi(\theta_1)(y|\cdot)$ denotes the $y^{th}$ row of $\Pi(\theta_1)$. Therefore, to establish the Lipschitz continuity of the matrix $\Pi(\theta)$ with respect to $\theta$, it suffices to show that $Q^{\theta}$ is Lipschitz continuous in $\theta$.

For any $\theta$, the function $Q^{\theta}$ is the unique fixed point of the Bellman operator $T^{\theta}$. This operator forms a $\beta$-contraction mapping under the supremum norm, given by
$$
T^{\theta} \, Q(x,a) \triangleq \langle \theta, f(x,a) \rangle_{\R^{\sX} \times \cH} + \beta \, \sum_{y \in \sX} V(y) \, p(y|x,a,\mu_E),
$$
where $V(y) \triangleq \log \sum_{a \in \sA} e^{Q(y,a)}$. Moreover,
\begin{align*} \|Q^{\theta_1} - Q^{\theta_2}\|_{\infty} 
&= \|T^{\theta_1} \, Q^{\theta_1} - T^{\theta_2} \, Q^{\theta_2}\|_{\infty} \\
&\leq \|T^{\theta_1} \, Q^{\theta_1} - T^{\theta_1} \, Q^{\theta_2}\|_{\infty}  + \|T^{\theta_1} \, Q^{\theta_2} - T^{\theta_2} \, Q^{\theta_2}\|_{\infty} \\
&\leq  \beta \, \|Q^{\theta_1} - Q^{\theta_2}\|_{\infty} + \| \langle \theta_1, f(\cdot,\cdot) \rangle_{\R^{\sX} \times \cH} -  \langle \theta_2, f(\cdot,\cdot) \rangle_{\R^{\sX} \times \cH}\|_{\infty} \\
&\leq  \beta \, \|Q^{\theta_1} - Q^{\theta_2}\|_{\infty}  +  \| \theta_1 - \theta_2 \|_{\R^{\sX} \times \cH} \, \|\| f(\cdot,\cdot) \|_{\R^{\sX} \times \cH}\|_{\infty}  \\
&= \beta \, \|Q^{\theta_1} - Q^{\theta_2}\|_{\infty} + K \, \| \theta_1 - \theta_2 \|_{\R^{\sX} \times \cH},
\end{align*}
where $K \triangleq \max_{(x,a) \in \sX \times \sA} \| f(x,a) \|_{\R^{\sX} \times \cH} < \infty$. Hence, 
$$
\|Q^{\theta_1} - Q^{\theta_2}\|_{\infty} \leq \frac{K}{1-\beta} \, \| \theta_1 - \theta_2 \|_{\R^{\sX} \times \cH},
$$ 
which establishes the Lipschitz continuity of $Q^{\theta}$ with respect to the sup-norm. Consequently, we obtain the bound 
\begin{align*}
\max_{y \in \sX} \|Q^{\theta_1}(y,\cdot) - Q^{\theta_2}(y,\cdot)\|_{\infty} =  \|Q^{\theta_1} - Q^{\theta_2}\|_{\infty}  \leq \frac{K}{1-\beta} \, \| \theta_1 - \theta_2 \|_{\R^{\sX} \times \cH}.
\end{align*}
We now synthesize the preceding results to establish the Lipschitz continuity of $\nabla Q^{\theta}$, thereby proving the smoothness of $Q^{\theta}$.
\small
\begin{align*}
&\max_{(x,a) \in \sX \times \sA} \|\nabla Q^{\theta_1}(x,a) - \nabla Q^{\theta_2}(x,a)\|_{\R^{\sX} \times \cH}  \\
&= \max_{(x,a) \in \sX \times \sA} \|(I -\beta \, P \, \Pi(\theta_1))^{-1} \, F(x,a|\cdot) - (I -\beta \, P \, \Pi(\theta_2))^{-1} \, F(x,a|\cdot)\|_{\R^{\sX} \times \cH} \\
&= \max_{(x,a) \in \sX \times \sA} \bigg\|\sum_{(y,b) \in \sX \times \sA} \!\! (I -\beta \, P \, \Pi(\theta_1))^{-1}(x,a|y,b) \, F(y,b|\cdot) -  \!\!\!\!\!\sum_{(y,b) \in \sX \times \sA} \!\! (I -\beta \, P \, \Pi(\theta_2))^{-1}(x,a|y,b) \, F(y,b|\cdot) \bigg\|_{\R^{\sX} \times \cH} \\
&\leq \max_{(x,a) \in \sX \times \sA} \hspace{-10pt} \sum_{(y,b) \in \sX \times \sA} \bigg| (I -\beta \, P \, \Pi(\theta_1))^{-1}(x,a|y,b) - (I -\beta \, P \, \Pi(\theta_2))^{-1}(x,a|y,b) \bigg| \, \|f(y,b)\|_{\R^{\sX} \times \cH} \\
&= \max_{(y,b) \in \sX \times \sA} \|f(y,b)\|_{\R^{\sX} \times \cH} \, \|(I -\beta \, P \, \Pi(\theta_1))^{-1} -(I -\beta \, P \, \Pi(\theta_2))^{-1}\|_{\infty} \\
&\leq K \, \|(I -\beta \, P \, \Pi(\theta_1))^{-1}-(I -\beta \, P \, \Pi(\theta_2))^{-1}\|_{\infty}.
\end{align*}
\normalsize
Note that $B^{-1} - A^{-1} = A^{-1} (A-B) B^{-1}$, and so, $$\|A^{-1} - B^{-1}\|_{\infty} \leq \|A^{-1}\|_{\infty} \, \|A-B\|_{\infty} \, \|B^{-1}\|_{\infty} \,.$$ This implies that 
\begin{align*}
&\|(I -\beta \, P \, \Pi(\theta_1))^{-1}-(I -\beta \, P \, \Pi(\theta_2))^{-1}\|_{\infty} \\
&\leq \|(I -\beta \, P \, \Pi(\theta_1))^{-1}\|_{\infty} \, \|\beta \, P \, \Pi(\theta_1) - \beta \, P \, \Pi(\theta_2)\|_{\infty}  \|(I -\beta \, P \, \Pi(\theta_2))^{-1}\|_{\infty} \\
&\leq \|(I -\beta \, P \, \Pi(\theta_1))^{-1}\|_{\infty} \, \beta \, \|P\|_{\infty} \|\Pi(\theta_1) - \Pi(\theta_2)\|_{\infty}  \|(I -\beta \, P \, \Pi(\theta_2))^{-1}\|_{\infty} \\ 
&= \bigg \|\sum_{k=0}^{\infty} \beta^k \, (P \, \Pi(\theta_1))^{k} \bigg\|_{\infty}  \, \beta \, \|P\|_{\infty} \|\Pi(\theta_1) - \Pi(\theta_2)\|_{\infty}   \bigg \|\sum_{k=0}^{\infty} \beta^k \, (P \, \Pi(\theta_2))^{k} \bigg\|_{\infty} \\
&\leq \sum_{k=0}^{\infty} \beta^k \,  \|(P \, \Pi(\theta_1))\|_{\infty}^k  \, \beta \, \|P\|_{\infty} \|\Pi(\theta_1) - \Pi(\theta_2)\|_{\infty} \sum_{k=0}^{\infty} \beta^k \, \|(P \, \Pi(\theta_2))\|_{\infty}^k \\
&\leq \frac{\beta}{(1-\beta)^2} \, \|\Pi(\theta_1) - \Pi(\theta_2)\|_{\infty},
\end{align*}
where the last inequality follows from the facts that $\|P\|_{\infty} \leq 1$ as $P$ is a transition matrix, and similarly, $\|(P \, \Pi(\theta_2))\|_{\infty} \leq 1$ as $P \, \Pi(\theta_2)$ also represents a transition matrix. In view of this,
\begin{align*}
\max_{(x,a) \in \sX \times \sA} \|\nabla Q^{\theta_1}(x,a) - \nabla Q^{\theta_2}(x,a)\|_{\R^{\sX} \times \cH}  & \leq \frac{K \, \beta}{(1-\beta)^2} \, \|\Pi(\theta_1) - \Pi(\theta_2)\|_{\infty} \\
&\leq \frac{K \, \beta}{(1-\beta)^2} \,\max_{y \in \sX} \|Q^{\theta_1}(y,\cdot) - Q^{\theta_2}(y,\cdot)\|_{\infty} \\
&\leq \frac{K^2 \, \beta}{(1-\beta)^3} \, \| \theta_1 - \theta_2 \|_{\R^{\sX} \times \cH}.
\end{align*}
This completes the proof that $Q^{\theta}(x,a)$ is a smooth function of $\theta$ for all state-action pairs $(x,a) \in \sX \times \sA$.

We now prove that the value function $V^{\theta}(x)$ is also smooth with respect to the parameter $\theta$ for all states $x \in \sX$. Indeed,

\small
\begin{align*}
&\max_{x \in \sX} \|\nabla V^{\theta_1}(x) - \nabla V^{\theta_2}(x)\|_{\R^{\sX} \times \cH} \\ 
& = \max_{x \in \sX} \|\Pi(\theta_1) \, \nabla Q^{\theta_1}(x|\cdot) - \Pi(\theta_2) \, \nabla Q^{\theta_2}(x|\cdot)\|_{\R^{\sX} \times \cH} \\
&=\max_{x \in \sX} \bigg \|\sum_{(y,b) \in \sX \times \sA} \Pi(\theta_1)(x|y,b) \, \nabla Q^{\theta_1}(y,b|\cdot)-  \sum_{(y,b) \in \sX \times \sA} \Pi(\theta_2)(x|y,b) \, \nabla Q^{\theta_2}(y,b|\cdot)     \bigg\|_{\R^{\sX} \times \cH} \\
&\leq \max_{x \in \sX} \bigg \|\sum_{(y,b) \in \sX \times \sA} \Pi(\theta_1)(x|y,b) \, \nabla Q^{\theta_1}(y,b|\cdot)  -  \sum_{(y,b) \in \sX \times \sA} \Pi(\theta_1)(x|y,b) \, \nabla Q^{\theta_2}(y,b|\cdot)     \bigg\|_{\R^{\sX} \times \cH} \\
&\hspace{1em}+ \max_{x \in \sX} \bigg \|\sum_{(y,b) \in \sX \times \sA} \Pi(\theta_1)(x|y,b) \, \nabla Q^{\theta_2}(y,b|\cdot)  -  \sum_{(y,b) \in \sX \times \sA} \Pi(\theta_2)(x|y,b) \, \nabla Q^{\theta_2}(y,b|\cdot)     \bigg\|_{\R^{\sX} \times \cH} \\
&\leq \|\Pi(\theta_1)\|_{\infty} \, \max_{(y,b) \in \sX \times \sA} \|\nabla Q^{\theta_1}(y,b|\cdot) - \nabla Q^{\theta_2}(y,b|\cdot)\|_{\R^{\sX} \times \cH}  +\|\Pi(\theta_1) - \Pi(\theta_2)\|_{\infty} \, \max_{(y,b) \in \sX \times \sA} \|\nabla Q^{\theta_2}(y,b|\cdot)\|_{\R^{\sX} \times \cH} \\
&\leq \frac{K^2 \, \beta}{(1-\beta)^3} \, \| \theta_1 - \theta_2 \|_{\R^{\sX} \times \cH} + \frac{K^2 }{(1-\beta)^2} \, \| \theta_1 - \theta_2 \|_{\R^{\sX} \times \cH},
\end{align*}
\normalsize
where $\|\Pi(\theta_1)\|_{\infty} \leq 1$ as $\Pi(\theta_1)$ being a transition matrix and
\begin{align*}
\max_{(y,b) \in \sX \times \sA} \|\nabla Q^{\theta_2}(y,b|\cdot)\|_{\R^{\sX} \times \cH}
&= \max_{(y,b) \in \sX \times \sA} \|(I -\beta \, P \, \Pi(\theta))^{-1} \, F(y,b|\cdot)\|_{\R^{\sX} \times \cH} \\
&\leq \|(I -\beta \, P \, \Pi(\theta))^{-1}\|_{\infty} \, \max_{(y,b) \in \sX \times \sA} \|f(y,b)\|_{\R^{\sX} \times \cH} \\
&\leq \frac{K}{1-\beta}.
\end{align*}
Substituting these bounds yields
\begin{align*}
&\|\nabla \V(\theta_1) - \nabla \V(\theta_2)\|_{\R^{\sX} \times \cH} \\
&\leq \hspace{-10pt} \sum_{(x,a) \in \sX \times \sA} \Big(\|\nabla Q^{\theta_1}(x,a) - \nabla Q^{\theta_2}(x,a)\|_{\R^{\sX} \times \cH} + \|\nabla V^{\theta_1}(x) - \nabla V^{\theta_2}(x)\|_{\R^{\sX} \times \cH}\Big) \, \gamma_{\pi_E}(x,a) \\
&\leq \frac{K^2}{(1-\beta)^2} \, \left(\frac{2  \, \beta}{1-\beta} + 1 \right) \, \| \theta_1 - \theta_2 \|_{\R^{\sX} \times \cH}.
\end{align*}
This establishes the Lipschitz continuity of $\nabla V^{\theta}$, completing the proof.
\end{proof}

The following result establishes the consistency of Algorithm~\ref{GD} in view of Proposition~\ref{prop}.

\begin{theorem}
Suppose the step-size $\gamma$ in the gradient ascent algorithm satisfies  $0<\gamma \leq \frac{1}{L}$. Then,
$$
\|\nabla \V(\theta_k)\| \rightarrow 0~~\text{as}~~ k \rightarrow \infty.
$$

\end{theorem}

\begin{proof}
Although this is a standard result in nonlinear optimization (following directly from the descent lemma), we provide a concise proof for completeness.

Since $\V$ is $L$-smooth, the function $-\V$ is also $L$-smooth. Consequently, for any $\theta$ and $\tilde \theta$,  we have 
$$
-\V(\tilde \theta) + \V(\theta) \leq \langle -\nabla \V(\theta), \tilde \theta - \theta \rangle + \frac{L}{2} \|\tilde \theta-\theta\|^2.
$$
Substituting $\tilde{\theta} = \theta_{k+1}$ and $\theta = \theta_k$ into the inequality, and recalling the gradient ascent update rule $$\theta_{k+1} - \theta_k = \gamma \, \nabla \V(\btheta_{k}),$$ it follows that
\begin{align*}
-\V(\theta_{k+1}) + \V(\theta_k) &\leq \langle -\nabla \V(\theta_k), \gamma \, \nabla \V(\btheta_{k}) \rangle + \frac{L}{2} \|\gamma \, \nabla \V(\btheta_{k})\|^2 \\
&= \left( -\gamma + \frac{L \gamma^2}{2} \right) \|\nabla V(\theta_k)\|^2 \triangleq -\alpha \, \|\nabla V(\theta_k)\|^2,
\end{align*} 
where $-\alpha > 0$ as $\gamma \leq \frac{1}{L}$. Therefore, we obtain
$$
\|\nabla V(\theta_k)\|^2 \leq \frac{1}{\alpha} \, (\V(\theta_{k+1}) - \V(\theta_{k})). 
$$
Summing both sides of the inequality over $k = 0, \dots, T$  leads to the bound
\begin{align*}
\sum_{k=0}^T \|\nabla V(\theta_k)\|^2  &\leq \frac{1}{\alpha} \left( \V(\theta_{T+1}) - \V(\theta_{0}) \right) \\ &\leq \frac{1}{\alpha} \left( \sup_{\theta} \V(\theta) - \V(\theta_0) \right) < \infty. 
\end{align*}
This bound holds for all $T \in \mathbb{N}$, which implies that the infinite series is convergent
$$
\sum_{k=0}^{\infty} \|\nabla V(\theta_k)\|^2 < \infty.
$$
Thus, the gradient norms vanish asymptotically, $|\nabla \V(\theta_k)| \to 0$ as $k \to \infty$.
\end{proof}

\tb{Having established the framework for stationary MFGs, we now extend to the non-stationary finite-horizon setting, where the log-likelihood interpretation is no longer available.}

\section{Extension to Non-Stationary Setting}

We now extend the framework to the non-stationary, finite-horizon setting. Throughout this section, we use boldface notation (e.g., $\btheta, \bpi, \bmu$) to denote time-indexed collections. While the Lagrangian relaxation and soft Bellman machinery carry over, a fundamental difference emerges: the relaxation can no longer be reduced to a log-likelihood formulation. We identify precisely where the stationary argument breaks down and develop the necessary modifications. \tb{We begin by formulating the non-stationary MFG problem.}

A discrete-time, finite-horizon discounted non-stationary MFG is defined by the same components $(\sX,\sA,p,r,\beta)$ as its stationary counterpart. The essential distinction between the two formulations resides in the characterization of the equilibrium. Specifically, given a prescribed measure flow ${\bmu} \triangleq \{\mu(t)\}_{t=0}^{T-1}$, where $T$ denotes the finite time horizon, the state evolves according to the rule $x(0) \sim \mu(0)$ and $x(t+1) \sim p(\cdot|x(t),a(t),\mu(t))$ for $t=0,\ldots,T-2$. Under this setup, the instantaneous reward at each time $t=0,\ldots,T-1$ is $r(x(t),a(t),\mu(t))$. In the non-stationary setting, a policy is a time-indexed collection $\bpi \triangleq \{\pi_t\}_{t=0}^{T-1}$, where each $\pi_t$ is a conditional distribution over actions $\sA$ given states $\sX$. Consequently, the finite-horizon discounted return for any policy $\bpi$ is defined as
$$
J_{\bmu}(\bpi) \triangleq E^{\bpi,\bmu} \left[ \sum_{t=0}^{T-1} \beta^t \, r(x(t),a(t),\mu(t)) \right].
$$

To define the concept of equilibrium in this MFG model, we introduce two mappings. The first mapping $\Psi : \Pnew(\sX)^T \rightarrow 2^{\Pi^T}$ is set-valued and defined by
$$\Psi(\bmu) = \{\hat{\bpi} \in \Pi^T: J_{\bmu}(\hat{\bpi}) = \sup_{\bpi} J_{\bmu}(\bpi) \}$$
represents the optimal policies for a specified $\bmu$.
On the other hand, $\Lambda : \Pi^T \to \Pnew(\sX)^T$ is a single-valued and maps any policy $\bpi \in \Pi^T$ to the measure flow $\bmu_{\bpi}$ whose elements are defined recursively as follows:
$$
\mu_{\bpi}(t+1)(\cdot) = \sum_{(x,a) \in \sX \times \sA} p(\cdot|x,a,\mu_{\bpi}(t)) \, \pi_t(a|x) \, \mu_{\bpi}(t)(x).
$$

\begin{definition}
A pair $(\bpi_*,\bmu_*) \in \Pi^T \times \Pnew(\sX)^T$ is called a mean-field equilibrium if $\bpi_* \in \Psi(\bmu_*)$ and $\bmu_* = \Lambda(\bpi_*)$. That is, $\bpi^*$ is  optimal with respect to the measure flow $\bmu^*$, and $\bmu^*$ is the collection of state distributions under the policy $\bpi^*$.
\end{definition}

    In the non-stationary setting,  the expert provides a dataset $\mathcal{D}$ consisting of $d$ independent trajectories
$$
\D=\left\{\left(x_i(t),a_i(t)\right)_{t=0}^{T-1}\right\}_{i=1}^d,
$$
where each trajectory is generated according to a mean-field equilibrium $(\bpi_E, \bmu_E)$. Using this data, the measure flow $\bmu_E$ and the feature expectations for each $t=0,\ldots,T-1$ 
$$\langle \Phi \rangle_{\bpi_E,\bmu_E,t} \coloneqq E^{\bpi_E,\bmu_E}\left[ \,\Phi(x(t),a(t),\mu_E(t))\right] \in \cH$$
under the expert policy $\bpi_E$ can be estimated. Hence, we adopt the following assumption for the remainder of this section, which is similar to the stationary setting.

\begin{assumption}
The following quantities are known: (i) the measure flow $\bmu_E$, and (ii) the feature expectations $\langle \Phi \rangle_{\bpi_E, \bmu_E,t}$, $t=0,\ldots,T-1$, under the expert policy $\bpi_E$.
\end{assumption}


Now, we formulate the maximum causal entropy IRL problem for the non-stationary setup. To this end, the discounted causal entropy of a policy $\bpi$ is defined as
$$
H(\bpi) = \sum_{t=0}^{T-1} \beta^t E^{\bpi,\bmu_E} \left[-\log \, \pi(t)(a(t)|x(t)) \right].
$$
Given this, we formulate the kernel-based maximum discounted causal entropy IRL problem as the following constrained optimization:
\begin{align*}
&\mathbf{(OPT_2)} \ \text{maximize}_{\bpi \in \Pi^T} \ H(\bpi) ~~ \text{subject to:} \\
&\mu_{E}(t+1)(\cdot) = \sum_{(x,a) \in \sX \times \sA} p(\cdot|x,a,\mu_{E}(t)) \, \pi(t)(a|x) \, \mu_{E}(t)(x) \\
&E^{\bpi,\bmu_E}[\Phi(x(t),a(t),\mu_E(t))] = \langle \Phi \rangle_{\bpi_E,\bmu_E,t} \\
&\forall t=0,\ldots,T-1.
\end{align*}

The optimal solution of \(\mathbf{(OPT_2)}\), together with the measure flow \(\bmu_E\), constitutes an equilibrium.

\begin{proposition} 
Let $\bpi^*$ be the solution of $\mathbf{\mathbf{(OPT_2)}}$. Then, the pair $(\bpi^*,\bmu_E)$ is a mean-field equilibrium.
\end{proposition}

The proof of this result is very similar to the proof of Proposition~\ref{pr2}, and so, we omit the details. Note that the first constraint in $\mathbf{(OPT_2)}$ actually states that under $\bmu_E$, the distribution of $x(t)$ is $\mu_E(t)$ for all $t=0,\ldots,T-1$. Hence, we can re-formulate $\mathbf{(OPT_2)}$ alternatively as follows:
\begin{align*}
&\mathbf{(OPT_2)} \ \text{maximize}_{\bpi \in \Pi^T} \ H(\bpi) ~~ \text{subject to:} \\
& \beta^t \, E^{\bpi,\bmu_E} [1_{\{x(t) = \cdot\}}] = \beta^t \, \mu_E(t)(\cdot) \\
&\beta^t \, E^{\bpi,\bmu_E}[\Phi(x(t),a(t),\mu_E(t))] = \beta^t \, \langle \Phi \rangle_{\bpi_E,\bmu_E,t} \\
&\forall t=0,\ldots,T-1.
\end{align*}
Here, we introduce an extra \(\beta^t\) factor in front of the constraints. This modification does not alter the original problem, but it allows the Lagrangian relaxation to be interpreted as a finite-horizon entropy-regularized MDP with discounting.


Similar to the stationary case, we now analyze the Lagrangian relaxation of the optimization problem \(\mathbf{(OPT_2)}\) using its re-formulated version. In contrast to the stationary setting, however, the Lagrangian relaxation can no longer be transformed into a log-likelihood formulation. We explain this limitation in detail. We then develop an alternative solution approach based on Danskin's theorem.

We introduce the Lagrange multiplier $\btheta \triangleq (\blambda,\bh) \in (\R^{\sX})^T \times (\cH)^T$ and define the Lagrangian relaxation associated with $\mathbf{(OPT_2)}$ as
\begin{align*}
  \cG(\btheta)  \triangleq \max_{\bpi \in \Pi^T} \cL(\bpi,\btheta) &\triangleq \max_{\bpi \in \Pi^T}   H(\bpi) + \sum_{t=0}^{T-1}\left \langle \lambda(t), \beta^t \, E^{\bpi,\bmu_E}[1_{\{x(t)=\cdot\}}] - \beta^t \, \mu_E(t)(\cdot) \right \rangle_{\R^{\sX}} \\
 & +\sum_{t=0}^{T-1} \left \langle h(t), \beta^t \, E^{\bpi,\bmu_E}[\Phi(x(t),a(t),\mu_E(t))] - \beta^t \, \langle \Phi \rangle_{\bpi_E,\bmu_E,t} \right \rangle_{\cH}.
\end{align*}
Then,
$$
\mathbf{(OPT_2)} \leq \min_{\btheta} \cG(\btheta) \triangleq \cG(\btheta^*).
$$
Since the terms $\langle \lambda(t), \mu_E(t) \rangle_{\R^{\sX}}$ and $\langle h(t), \langle \Phi \rangle_{\bpi_E,\bmu_E,t} \rangle_{\cH}$ are independent of the policy $\bpi$, they can be dropped from the Lagrangian objective. This leads to the simplified problem:
\begin{align*}
\max_{\bpi \in \Pi^T} \text{ }  H(\bpi) + \sum_{t=0}^{T-1} \beta^t \, E^{\bpi,\bmu_E}[\lambda(t)(x(t))]   \sum_{t=0}^{T-1} \beta^t \, E^{\bpi,\bmu_E}[h(t)(x(t),a(t),\mu_E(t))].
\end{align*}
Note that the above problem is indeed an entropy regularized finite-horizon non-stationary discounted MDP with the non-stationary reward function 
$$r_{\btheta,t}(x,a) \triangleq \lambda(t)(x) + h(t)(x,a,\mu_E(t)).$$ 
The solution to this problem is given by the following soft Bellman optimality equations: for $t=0,\ldots,T-1$
\begin{align*}
Q^{\btheta,t}(x,a) &= r_{\btheta,t}(x,a) + \beta \, \sum_{y \in \sX} V^{\btheta,t+1}(y) \, p(y|x,a,\mu_E(t)) \\
V^{\btheta,t}(x) &= \log \sum_{a \in \sA} e^{Q^{\btheta,t}(x,a)} \triangleq \softmax_{a \in \sA} Q^{\btheta,t}(x,a),
\end{align*}
where we set $V_{\btheta,T} = 0$.
Then, it follows that
$$
\pi^{\btheta,t}(a|x) = e^{Q^{\btheta,t}(x,a)-V^{\btheta,t}(x)}
$$
is the optimal policy (see \cite{NeJoGo17}). 
Contrary to the stationary case, here $Q^{\btheta,t}$ and $V^{\btheta,t}$ are automatically Fr\'echet differentiable with respect to the parameter vector $\btheta$. This follows from the recursive relation between these functions due to soft-Bellman optimality conditions and the boundary condition $V^T=0$. In the stationary setting, instead of a recursive construction, we have a fixed point equation, and so, to establish Fr\'echet differentiability we need to invoke the implicit function theorem.

Now we mimic the argument used in the stationary case in order to derive a log-likelihood formulation corresponding to the above Lagrangian relaxation. We then explain why this formulation fails in the present setting. For any policy $\bpi$, let 
$$
\gamma_{\bpi,t}(x,a) \coloneqq \beta^t \, E^{\bpi,\bmu_E} \left[ 1_{\{(x(t),a(t)) = (x,a)\}} \right]
$$
denote the weighted state-action distribution at time $t$ under $(\bpi,\bmu_E)$.
We consider the following function:
$$
\V(\btheta) \triangleq \sum_{t=0}^{T-1} \sum_{(x,a) \in (\sX\times\sA)} \log \pi^{\btheta,t}(a|x) \, \gamma_{\bpi_E,t}(x,a).
$$
The following theorem characterizes the stationary points of $\V$.

\begin{theorem}
If $\btheta^*$ satisfies $\nabla \V(\btheta^*) = 0$, then

$$\sum_{t=0}^{T-1} \begin{bmatrix}
E^{\bpi^{\btheta},\bmu_E} \left[ 1_{\{x(t) = \cdot\}} \right] \\ E^{\bpi^{\btheta},\bmu_E} \left[ \Phi(x(t),a(t),\mu_E(t)) \right]
\end{bmatrix} = \sum_{t=0}^{T-1}\begin{bmatrix}
\mu_E(t) \\ \langle \Phi \rangle_{\bpi_E,\bmu_E,t}.
\end{bmatrix}$$

\end{theorem}

\smallskip

\begin{proof}
 The proof follows analogous steps to Theorem~\ref{equivv3}, utilizing time-indexed quantities in place of stationary counterparts. Hence, we omit the details for brevity.
\end{proof}
 Unlike Theorem 4.2 in the stationary case, the vanishing gradient condition in Theorem 6.1 ensures only that the aggregate feature expectations match across all time steps, not that the individual time-indexed constraints of ${\mathbf{(OPT_2)}}$ are satisfied. This is the fundamental reason why the log-likelihood reformulation is unavailable in the non-stationary setting.

The preceding result shows that, unlike in the stationary case, $\mathbf{(OPT_2)}$ cannot be interpreted as a log-likelihood maximization problem. Consequently, a different approach is required to characterize the solution. We proceed by invoking Danskin's theorem \cite[Proposition B.25]{Ber99}.


Note that $\cG$, the objective function of the Lagrangian relaxation, is convex, since $\cL$ is linear in $\btheta$. Moreover, the set $\Pi^T$ is compact and for any $\btheta$, $\cL(\cdot,\btheta)$ has a unique minimizer $\bpi^{\btheta}$ over $\Pi^T$ since the problem is equivalent to the entropy regularized MDP problem. Hence,  the conditions of Danskin's theorem \cite[Proposition B.25]{Ber99}  are satisfied for $\cL$. This implies that $\cG(\btheta)$ is Fr\'echet differentiable with gradient $\nabla_{\btheta} \cG(\btheta) = \nabla_{\btheta} \cL(\bpi^{\btheta},\btheta)$. Since $\cG$ is also convex, we can invoke gradient descent to compute its minimum.
To be able to apply a constant step-size gradient descent algorithm for finding the minimum point of the function $\cG(\theta)$, we need to establish that $\cG$ is $L$-smooth for some $L>0$. Before addressing this, we first examine the structure of the gradient of $\cG$.

The gradient of $\cG$ by Danskin's theorem is given by $\nabla_{\btheta} \cG(\btheta) = \nabla_{\btheta} \cL(\bpi^{\btheta},\btheta)$ or more explicitly
$$
\begin{bmatrix} \langle f_0 \rangle_{\bpi^{\btheta},\bmu_E} - \langle f_0 \rangle_{\bpi_E,\bmu_E} \\ \vdots \\ \langle f_{T-1} \rangle_{\bpi^{\btheta},\bmu_E} - \langle f_{T-1} \rangle_{\bpi_E,\bmu_E} \end{bmatrix}.
$$
Hence, when $\nabla_{\btheta^*} \cG(\btheta^*) = 0$ at the optimal solution $\btheta^*$, the corresponding $\bpi^{\btheta^*}$ becomes feasible for ${\mathbf{(OPT_2)}}$. This leads to
\lk{\begin{align*}
\mathbf{(OPT_2)} &= \max_{\bpi} \min_{\btheta} \cL(\bpi,\btheta) 
\leq \min_{\btheta} \max_{\bpi} \cL(\bpi,\btheta) = \min_{\btheta} \cG(\btheta) \\
&= \max_{\bpi} \cL(\bpi,\btheta^*) = \cG(\btheta^*) = \cL(\bpi^{\btheta^*}, \btheta^*) \\
&\leq \mathbf{(OPT_2)} \quad \text{(since $\bpi^{\btheta^*}$ is feasible for $\mathbf{(OPT_2)}$).}
\end{align*}}
Therefore, 
\begin{align*}
\bpi^{\btheta^*} \in \argmax  \mathbf{(OPT_2)}.
\end{align*}

Now let us explain how one can compute the gradient of $\cG$. It is assumed that $\{\langle f_t \rangle_{\pi_E,\mu_E}\}_{t=0}^{T-1}$ is known \textit{a priori}. Consequently, computing the gradient \( \nabla \mathcal{G}(\btheta) \) reduces to evaluating the collection \( \{\langle f_t \rangle_{\pi^{\boldsymbol{\theta}}, \mu_E}\}_{t=0}^{T-1} \). Here, for each $t = 0, \ldots, T - 1$, we define
\(
f_t(x, a) \coloneqq \bigl(e_x, \Phi(x, a, \mu_E(t))\bigr)^{\top}
\in \mathbb{R}^{\sX} \times \mathcal{H},
\)
the time-indexed analogue of the feature function $f$ introduced in Section~IV.
This can be carried out via the soft Bellman optimality equations. Specifically, one first computes the sequences \( \{V^{\boldsymbol{\theta},t}\}_{t=0}^{T-1} \) and \( \{Q^{\boldsymbol{\theta},t}\}_{t=0}^{T-1} \) backward in time. Then, the corresponding policies \( \{\pi^{\boldsymbol{\theta},t}\}_{t=0}^{T-1} \), given by \( \pi^{\boldsymbol{\theta},t} = e^{,Q^{\boldsymbol{\theta},t} - V^{\boldsymbol{\theta},t}} \), are obtained, and the associated state distributions are computed forward in time. Equipped with this method, we now introduce gradient descent algorithm to locate minimum of $\mathcal{G}(\btheta)$.

\begin{algorithm}[H]
	\caption{\lk{Gradient Descent}}
	\label{GD2}
	\begin{algorithmic}[1]
\STATE \textbf{Input:} initial parameter $\btheta_0$, stepsize $\gamma > 0$, iterations $K$
\STATE Initialize $\btheta \leftarrow \btheta_0$
\FOR{$k = 0,\ldots,K-1$}
    \STATE $\btheta_{k+1} \leftarrow \btheta_k - \gamma \nabla \mathcal{G}(\btheta_k)$
\ENDFOR
\STATE \textbf{Return:} $\btheta_K$ and $\pi^{\btheta_K}$
\end{algorithmic}

\end{algorithm}

\begin{remark}
Although it is not possible to reformulate \(\mathbf{(OPT_2)}\) as a log-likelihood maximization problem in the non-stationary case (as can be done in the stationary setting), it is still useful to compare the corresponding gradients. In the stationary case, the gradient of \(\mathcal V\) is
$$
\nabla \V(\theta) = \langle f \rangle_{\pi^{\theta},\mu_E} - \langle f \rangle_{\pi_E,\mu_E}.
$$
Similarly, the gradient of \(\mathcal G\) in the non-stationary case is
$$
\nabla \cG(\btheta) = \begin{bmatrix} \langle f_0 \rangle_{\bpi^{\btheta},\bmu_E} - \langle f_0 \rangle_{\bpi_E,\bmu_E} \\ \vdots \\ \langle f_{T-1} \rangle_{\bpi^{\btheta},\bmu_E} - \langle f_{T-1} \rangle_{\bpi_E,\bmu_E} \end{bmatrix}.
$$
In both cases, the gradient is simply the difference between the feature expectations under the current policy and those under the expert policy. Hence, both methods update the parameters in the direction of the constraint violation. Because of the structural properties of \(\mathcal V\) ($L$-smoothness) and \(\mathcal G\) (convexity and $L$-smoothness), gradient ascent or descent procedure eventually leads to a point where the gradient -- equivalently, the constraint violation -- vanishes. Owing to the dual structure of the problem, such a point corresponds to the optimal solution of the original (primal) optimization problem.

Therefore, the main advantage of the stationary formulation is interpretational: it admits a log-likelihood interpretation that connects to the well-established maximum likelihood framework, providing clear statistical meaning to the parameters.
\end{remark}

We now establish the smoothness of the dual function $\mathcal G(\boldsymbol\theta)$. To that end, let us specify the norms used in the proof. We write
\[
\boldsymbol\theta
=
\big(\lambda(0),\ldots,\lambda(T-1),
h(0),\ldots,h(T-1)\big).
\]
Each $\lambda(t)\in\mathbb R^{\sX}$ is equipped with the Euclidean norm $\|\cdot\|_2$, and each $h(t)\in\mathcal H$ with the Hilbert norm $\|\cdot\|_{\mathcal H}$. Define the product norm
\[
\|\boldsymbol\theta\|^2
=
\sum_{t=0}^{T-1}
\Big(
\|\lambda(t)\|_2^2
+
\|h(t)\|_{\mathcal H}^2
\Big).
\]
For policies $\bpi,\bpi'$,
\[
\|\bpi-\bpi'\|_{1}
=
\max_{t,x}
\sum_{a}
|\pi^t(a|x)-\pi'^t(a|x)|.
\]
Let
$
\sup_{x,a,t}
\|\Phi(x,a,\mu_E(t))\|_{\mathcal H}
\triangleq
B_\Phi.
$

\begin{proposition}
Under the above conventions, the dual function
\[
\mathcal G(\boldsymbol\theta)
=
\max_{\boldsymbol\pi\in\Pi^T}
\mathcal L(\boldsymbol\pi,\boldsymbol\theta)
\]
is convex, continuously differentiable, and $L$-smooth:
\[
\|\nabla\mathcal G(\boldsymbol\theta)
-
\nabla\mathcal G(\boldsymbol\theta')\|
\le
L
\|\boldsymbol\theta-\boldsymbol\theta'\|
\]
where 
\[ L = \frac{T(1+B_\Phi)^2}{(1-\beta)^2}. \]
\end{proposition}

\begin{proof}
Recall
\[
r_{\boldsymbol\theta,t}(x,a)
=
\lambda(t)(x)
+
h(t)(x,a,\mu_E(t)).
\]
Then
\begin{align*}
|r_{\boldsymbol\theta,t}(x,a)
-
r_{\boldsymbol\theta',t}(x,a)|
\le
|\lambda(t)(x)-\lambda'(t)(x)| +
\|h(t)-h'(t)\|_{\mathcal H}
\|\Phi(x,a,\mu_E(t))\|_{\mathcal H}.
\end{align*}
Thus
\[
\|r_{\boldsymbol\theta,t}
-
r_{\boldsymbol\theta',t}\|_\infty
\le
(1+B_\Phi)
\|\boldsymbol\theta-\boldsymbol\theta'\|.
\]
Let $C:=1+B_\Phi$. By soft-Bellman recursion:
\[
Q^{\boldsymbol\theta,t}
=
r_{\boldsymbol\theta,t}
+
\beta p_t V^{\boldsymbol\theta,t+1},
\]
where $p_t$ is the transition probability at time $t$, hence
\[
\|p_t V\|_\infty \le \|V\|_\infty.
\]
Therefore
\[
\|Q^{\boldsymbol\theta,t}
-
Q^{\boldsymbol\theta',t}\|_\infty
\le
C\|\boldsymbol\theta-\boldsymbol\theta'\|
+
\beta
\|V^{\boldsymbol\theta,t+1}
-
V^{\boldsymbol\theta',t+1}\|_\infty.
\]
Since log-sum-exp is $1$-Lipschitz in $\|\cdot\|_\infty$,
\[
\|V^{\boldsymbol\theta,t}
-
V^{\boldsymbol\theta',t}\|_\infty
\le
\|Q^{\boldsymbol\theta,t}
-
Q^{\boldsymbol\theta',t}\|_\infty.
\]
Backward induction yields
\[
\|Q^{\boldsymbol\theta,t}
-
Q^{\boldsymbol\theta',t}\|_\infty
\le
\frac{C}{1-\beta}
\|\boldsymbol\theta-\boldsymbol\theta'\|.
\]
Note that the $l_{\infty} \to l_1$-norm of Jacobian of softmax is less than $1$. Therefore, by mean-value theorem
\[
\|\boldsymbol\pi^{\boldsymbol\theta}
-
\boldsymbol\pi^{\boldsymbol\theta'}\|_{1}
\le
\frac{C}{1-\beta}
\|\boldsymbol\theta-\boldsymbol\theta'\|.
\]
For each $t$,
\begin{align*}
\nabla_{\lambda(t)}
\mathcal G(\boldsymbol\theta)
&=
\beta^t
E^{\boldsymbol\pi^{\boldsymbol\theta},\bmu_E}
[1_{\{x(t) = \cdot\}}] - \beta^t
E^{\bpi_E,\bmu_E}
[1_{\{x(t) = \cdot\}}] \\
\nabla_{h(t)}
\mathcal G(\boldsymbol\theta) 
&=
\beta^t
E^{\boldsymbol\pi^{\boldsymbol\theta},\bmu_E}
[\Phi(x(t),a(t),\mu_E(t))] - \beta^t
E^{\bpi_E,\bmu_E}
[\Phi(x(t),a(t),\mu_E(t))].
\end{align*}

For bounded $\psi$, by tensorization of the total variation distance, whereby the total variation between product measures is bounded by the sum of marginal total variations, we have
\begin{align*}
&\bigg|\mathbb E^{\boldsymbol\pi,\bmu_E}\!\left[\sum_{t=0}^{T-1}\psi(x(t),a(t),\mu_E(t))\right]
-
\mathbb E^{\boldsymbol\pi',\bmu_E}\!\left[\sum_{t=0}^{T-1}\psi(x(t),a(t),\mu_E(t))\right]\!\bigg|
\le
T \|\psi\|_{\infty}
\|\boldsymbol\pi-\boldsymbol\pi'\|_1.
\end{align*}
Hence
\[
\|\nabla\mathcal G(\boldsymbol\theta)
-
\nabla\mathcal G(\boldsymbol\theta')\|
\le
\frac{T C^2}{(1-\beta)^2}
\|\boldsymbol\theta-\boldsymbol\theta'\|.
\]
Since $C=1+B_\Phi$,
\[
L
=
\frac{T(1+B_\Phi)^2}{(1-\beta)^2}.
\]

\end{proof}

The following result establishes the consistency of Algorithm~\ref{GD2} in view of the previous result. The result is well-known and so we omit the details.

\begin{theorem}
Suppose the step-size $\gamma$ in the gradient descent algorithm satisfies  $0<\gamma \leq \frac{1}{L}$. Then,
\begin{align*}
\cG(\btheta_k) - \cG(\btheta^*) &\leq \frac{\|\btheta_0-\btheta^*\|^2}{2\gamma k} \\   
\intertext{and}
\|\nabla \cG(\theta_k)\| &\rightarrow 0~~\text{as}~~ k \rightarrow \infty.
\end{align*}
\end{theorem}

\section{Numerical Example}

We validate our framework on a mean-field traffic routing game 
whose expert policy exhibits state-dependent preference reversal. A population of drivers chooses between a shorter main road and a longer alternative route.

The state space $\sX = \{0,1,2,3\}$ encodes four congestion levels (Light, Light-Medium, Medium-Heavy, Heavy) and the action space $\sA = \{0,1\}$ corresponds to route choice (Main, Alternative). The transition kernel has the form
\begin{equation*}
p(y|x,a,\mu) = p_{\text{base}}(y|x,a) + \Delta p(y|\mu),
\end{equation*}
where the base matrices are given in Table~\ref{tab:transitions} and the mean-field correction
\begin{align*}
\Delta p(\text{Light}|\mu) &= -0.3 \cdot \mu(\text{Heavy}), \\
\Delta p(\text{Heavy}|\mu) &= +0.3 \cdot \mu(\text{Heavy})
\end{align*}
introduces a negative externality: higher population-level congestion makes escaping to light traffic harder system-wide. 

The main road exhibits a congestion trap ($p_{\text{base}}(\text{Heavy}|\text{Heavy},\text{Main}) = 0.50$, with only 5\% chance of reaching Light), while the alternative offers more favorable transitions out of congestion (25\% and 15\%, respectively).

\begin{table}[t]
\centering
\caption{Base transition matrices $P_{\text{base}}^{(a)}$ (rows: current state $x$; columns: next state $y$).}
\label{tab:transitions}
\begin{tabular}{lcccc}
\toprule
& Light & L-Med & M-Heavy & Heavy \\
\midrule
\multicolumn{5}{l}{\textit{Main road} ($a=0$)} \\
Light      & 0.70 & 0.20 & 0.08 & 0.02 \\
L-Med      & 0.35 & 0.45 & 0.15 & 0.05 \\
M-Heavy    & 0.15 & 0.30 & 0.40 & 0.15 \\
Heavy      & 0.05 & 0.15 & 0.30 & 0.50 \\
\midrule
\multicolumn{5}{l}{\textit{Alternative route} ($a=1$)} \\
Light      & 0.65 & 0.25 & 0.08 & 0.02 \\
L-Med      & 0.45 & 0.35 & 0.15 & 0.05 \\
M-Heavy    & 0.25 & 0.35 & 0.30 & 0.10 \\
Heavy      & 0.15 & 0.25 & 0.35 & 0.25 \\
\bottomrule
\end{tabular}
\end{table}

The expert policy and stationary mean-field distribution are
\begin{equation*}
\pi_E(\text{Main}|x) = (0.85,\; 0.70,\; 0.45,\; 0.20), \end{equation*} for $x = 0, 1, 2, 3$ respectively,
and
\begin{equation*}
\mu_E = (0.45,\; 0.30,\; 0.20,\; 0.05)^\top,
\end{equation*}
where the entries are indexed by increasing congestion level. The policy exhibits preference reversal: 85\% main-road preference in light traffic, transitioning to 80\% alternative preference in heavy traffic.

We compare two reward representations, both solved via Algorithm~1 with $\beta = 0.9$ and $K = 10{,}000$ iterations.
Feature expectations are computed analytically from $(\pi_E, \mu_E)$ with horizon truncation $T = 32$, eliminating sampling variance.

The \emph{linear baseline} (10 parameters) uses additive features
\begin{equation*}
r_\theta(x,a,\mu) = \sum_{i=0}^{3} \theta^x_i \mathbf{1}_{\{x=i\}}
+ \sum_{j=0}^{1} \theta^a_j \mathbf{1}_{\{a=j\}}
+ \sum_{k=0}^{3} \theta^\mu_k \mu_k,
\end{equation*}
with learning rate $\gamma = 0.005$. 

This additive specification tests whether state, action, and mean-field contributions suffice to explain the expert's behavior. Crucially, we exclude state--action interaction terms; including them would trivially enable perfect fitting but would obscure the structural limitation we aim to highlight. The kernel method, by contrast, captures these interactions implicitly through the kernel function.


The \emph{kernel method} (12 parameters) uses a Gaussian kernel $k(z,z') = \exp(-\|z-z'\|^2 / 2\sigma^2)$ with $\sigma = 0.5$ and $M = 8$ anchor points (all state--action pairs at $\mu_E$):
\begin{equation*}
r_\theta(x,a,\mu) = \lambda(x) + \sum_{i=1}^{8} \alpha_i k\big((x,a,\mu),\, (x_i, a_i, \mu_E)\big),
\end{equation*}
where $\theta = (\lambda, \alpha) \in \mathbb{R}^{4} \times \mathbb{R}^{8}$, with learning rate $\gamma = 0.003$.

\begin{table}[t]
\centering
\caption{Performance Comparison}
\label{tab:comparison}
\begin{tabular}{lccc}
\toprule
Method & Policy Error & Gradient Norm & Parameters \\
\midrule
Linear Baseline & 11.60\% & 0.037 & 10 \\
Kernel-Based & 0.10\% & 0.001 & 12 \\
\bottomrule
\end{tabular}
\end{table}

\begin{table}[t]
\centering
\caption{Learned Policies by State}
\label{tab:policies}
\begin{tabular}{lcccccc}
\toprule
& \multicolumn{2}{c}{Expert} & \multicolumn{2}{c}{Linear} & \multicolumn{2}{c}{Kernel} \\
State & Main & Alt & Main & Alt & Main & Alt \\
\midrule
Light & 0.85 & 0.15 & 0.73 & 0.27 & 0.85 & 0.15 \\
Light-Med & 0.70 & 0.30 & 0.67 & 0.33 & 0.70 & 0.30 \\
Med-Heavy & 0.45 & 0.55 & 0.58 & 0.42 & 0.45 & 0.55 \\
Heavy & 0.20 & 0.80 & 0.48 & 0.52 & 0.20 & 0.80 \\
\bottomrule
\end{tabular}
\end{table}

Tables~\ref{tab:comparison} and~\ref{tab:policies} and Figure~\ref{fig:convergence} summarize the comparison. The kernel method achieves near-exact recovery ($\|\pi_\theta - \pi_E\|_F = 0.001$, 0.10\% error), while the linear baseline converges to 11.60\%, that is, a 99.1\% reduction in error. Both methods have converged: the final gradient norms are 0.001 (kernel) and 0.037 (linear). The linear model's residual gradient reflects the infeasibility of the feature-matching constraints in $(\widehat{\text{OPT}}_1)$ within the additive reward class.

{\tb{
Table~\ref{tab:policies} reveals that the linear model captures the correct qualitative trend but cannot achieve the preference reversal, outputting $\pi(\text{Main}|\text{Heavy}) = 0.48$ versus the expert's $0.20$. Because the additive reward decomposes as $f(x) + g(a) + h(\mu)$, the action ranking $g(\text{Main}) > g(\text{Alternative})$ holds uniformly across states, precluding preference reversal. The learned parameters $\theta^a = [0.67, -0.67]$ impose a uniform main-road bias across all states, confirming the structural limitation. Figure~\ref{fig:convergence} shows the kernel method reaches 5\% error by iteration 2{,}627 and continues to decrease exponentially, consistent with the convergence guarantee of Theorem~5.1 and the $L$-smoothness established in Proposition~5.1. The linear baseline asymptotes to its structural floor, confirming that representational capacity (not optimization) is the bottleneck.}}


\begin{figure}[t]
\centering
\includegraphics[width=\columnwidth]{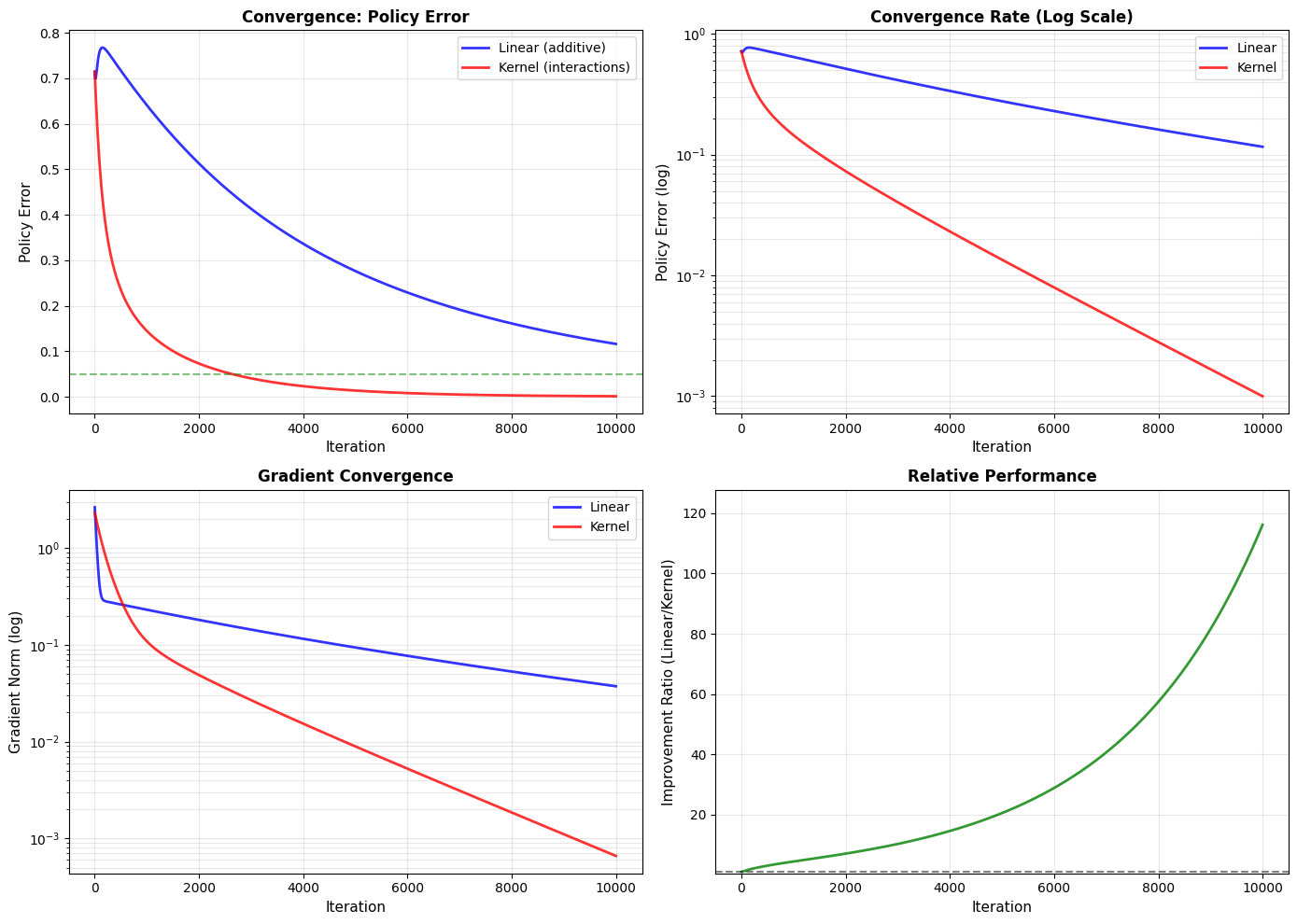}
\caption{Convergence comparison: (a) policy error; (b) log-scale view;
(c) gradient norms; (d) performance ratio (ratio of linear to kernel policy error over iterations).}
\label{fig:convergence}
\end{figure}

\section{Conclusion}
{\color{black}
In this work, we studied the IRL problem for infinite-horizon stationary MFGs through the lens of maximum causal entropy. By modeling the unknown reward function within a RKHS, we enabled the inference of rich and nonlinear reward structures directly from expert demonstrations. This addresses key limitations of existing IRL approaches that rely on linear reward models and finite-horizon settings. To solve the resulting problem, we introduced a Lagrangian relaxation that reformulates the IRL objective as an unconstrained log-likelihood maximization, which we tackled using a gradient ascent algorithm. We established the theoretical consistency of the proposed approach by proving the smoothness of the log-likelihood objective through the Fr\'echet differentiability of the associated soft Bellman operators. Finally, our numerical experiments on a mean-field traffic routing game validated the effectiveness of the method, demonstrating that the learned policies successfully replicate expert behavior.
}

We also extended the framework to the finite-horizon non-stationary setting, where we showed that the log-likelihood reformulation is structurally unavailable and developed an alternative gradient descent algorithm on the convex dual. Several directions merit future investigation. First, extending this framework to continuous-time formulations presents significant challenges. It would require reformulating the maximum causal entropy principle using path-space measures and solving the inverse problem within the context of Hamilton--Jacobi--Bellman equations and Fokker--Planck equations for the mean-field term. Second, a formal finite-sample analysis is a crucial next step. This would involve deriving high-probability concentration bounds for the error in the recovered policy as a function of the number of expert trajectories $d$, the trajectory length $T_i$, and the complexity of the RKHS (e.g., its metric entropy), building upon the $L$-smoothness established in Proposition~5.1.


\bibliographystyle{ieeetr}
\bibliography{references}

\begin{thebibliography}{10}

\bibitem{HuMaCa06}
M.~Huang, R.~Malham\'{e}, and P.~Caines, ``Large population stochastic dynamic
  games: Closed loop {M}c{K}ean-{V}lasov systems and the {N}ash certainty
  equivalence principle,'' {\em Communications in Information and Systems},
  vol.~6, pp.~221--252, 2006.

\bibitem{LaLi07}
J.~Lasry and P.~Lions, ``Mean field games,'' {\em Japan. J. Math.}, vol.~2,
  pp.~229--260, 2007.

\bibitem{BeFrPh13}
A.~Bensoussan, J.~Frehse, and P.~Yam, {\em Mean {F}ield {G}ames and {M}ean
  {F}ield {T}ype {C}ontrol {T}heory}.
\newblock Springer, New York, 2013.

\bibitem{CaDe13}
R.~Carmona and F.~Delarue, ``Probabilistic analysis of mean-field games,'' {\em
  SIAM J. Control Optim.}, vol.~51, no.~4, pp.~2705--2734, 2013.

\bibitem{7874135}
B.~Djehiche, A.~Tcheukam, and H.~Tembine, ``A mean-field game of evacuation in
  multilevel building,'' {\em IEEE Transactions on Automatic Control}, vol.~62,
  no.~10, pp.~5154--5169, 2017.

\bibitem{BaTeBa2016}
D.~Bauso, H.~Tembine, and T.~Ba\c{s}ar, ``Opinion dynamics in social networks
  through mean-field games,'' {\em SIAM Journal on Control and Optimization},
  vol.~54, no.~6, pp.~3225--3257, 2016.

\bibitem{9061051}
T.~Tanaka, E.~Nekouei, A.~R. Pedram, and K.~H. Johansson, ``Linearly solvable
  mean-field traffic routing games,'' {\em IEEE Transactions on Automatic
  Control}, vol.~66, no.~2, pp.~880--887, 2021.

\bibitem{WeBeRo05}
G.~Weintraub, L.~Benkard, and B.~Van~Roy, ``Oblivious equilibrium: A mean field
  approximation for large-scale dynamic games,'' in {\em Advances in Neural
  Information Processing Systems}, vol.~18, 2005.

\bibitem{AdJoWe15}
S.~Adlakha, R.~Johari, and G.~Weintraub, ``Equilibria of dynamic games with
  many players: Existence, approximation, and market structure,'' {\em Journal
  of Economic Theory}, vol.~156, pp.~269--316, 2015.

\bibitem{YiMeMeSh14}
H.~Yin, P.~Mehta, S.~Meyn, and U.~Shanbhag, ``Learning in mean-field games,''
  {\em IEEE Transactions on Automatic Control}, vol.~59, no.~3, pp.~629--644,
  2014.

\bibitem{LaPePeGiMuElGePi24}
M.~Lauriere, S.~Perrin, J.~Perolat, S.~Girgin, P.~Muller, R.~Elie, M.~Geist,
  and O.~Pietquin, {\em Learning in Mean Field Games: A Survey}.
\newblock preprint, arXiv:2205.12944, 2024.

\bibitem{AdCoBe22}
S.~Adams, T.~Cody, and P.~A. Beling, ``A survey of inverse reinforcement
  learning,'' {\em Artif. Intell. Rev.}, vol.~55, pp.~4307--4346, Aug. 2022.

\bibitem{KOPF201714902}
F.~K\"{o}pf, J.~Inga, S.~Rothfu\ss, M.~Flad, and S.~Hohmann, ``Inverse
  reinforcement learning for identification in linear-quadratic dynamic
  games,'' {\em IFAC-PapersOnLine}, vol.~50, no.~1, pp.~14902--14908, 2017.
\newblock 20th IFAC World Congress.

\bibitem{LIAN2022110524}
B.~Lian, W.~Xue, F.~L. Lewis, and T.~Chai, ``Inverse reinforcement learning for
  multi-player noncooperative apprentice games,'' {\em Automatica}, vol.~145,
  p.~110524, 2022.

\bibitem{9815022}
B.~Lian, V.~S. Donge, F.~L. Lewis, T.~Chai, and A.~Davoudi, ``Data-driven
  inverse reinforcement learning control for linear multiplayer games,'' {\em
  IEEE Transactions on Neural Networks and Learning Systems}, vol.~35, no.~2,
  pp.~2028--2041, 2024.

\bibitem{YaYeTrXuZh18}
J.~Yang, X.~Ye, R.~Trivedi, H.~Xu, and H.~Zha, ``Deep mean field games for
  learning optimal behavior policy of large populations,'' in {\em
  International Conference on Learning Representations}, 2018.

\bibitem{YaLiLiHu22}
Y.~Chen, L.~Zhang, J.~Liu, and S.~Hu, ``Individual-level inverse reinforcement
  learning for mean field games,'' in {\em Proceedings of the 21st
  International Conference on Autonomous Agents and Multiagent Systems}, AAMAS
  '22, pp.~253--262, 2022.

\bibitem{ChZhLiWi23}
Y.~Chen, L.~Zhang, J.~Liu, and M.~Witbrock, ``Adversarial inverse reinforcement
  learning for mean field games,'' in {\em Proceedings of the 22nd
  International Conference on Autonomous Agents and Multiagent Systems}, AAMAS
  '23, pp.~1088--1096, 2023.

\bibitem{chen2024meta}
Y.~Chen, X.~Lin, B.~Yan, L.~Zhang, J.~Liu, N.~Tan, and M.~Witbrock,
  ``Meta-inverse reinforcement learning for mean field games via probabilistic
  context variables,'' {\em Proceedings of the AAAI Conference on Artificial
  Intelligence}, vol.~38, no.~10, pp.~11407--11415, 2024.

\bibitem{LinAB19}
X.~Lin, S.~C. Adams, and P.~A. Beling, ``Multi-agent inverse reinforcement
  learning for certain general-sum stochastic games,'' {\em J. Artif. Intell.
  Res.}, vol.~66, pp.~473--502, 2019.

\bibitem{chandra2025}
R.~Chandra, H.~Karnan, N.~Mehr, P.~Stone, and J.~Biswas, ``Multi-agent inverse
  reinforcement learning in real world unstructured pedestrian crowds,'' in
  {\em IEEE/RSJ International Conference on Intelligent Robots and Systems
  (IROS)}, October 2025.

\bibitem{TaylorP09}
G.~Taylor and R.~Parr, ``Kernelized value function approximation for
  reinforcement learning,'' in {\em Proceedings of the 26th Annual
  International Conference on Machine Learning, {ICML} 2009, Montreal, Quebec,
  Canada, June 14-18, 2009} (A.~P. Danyluk, L.~Bottou, and M.~L. Littman,
  eds.), vol.~382 of {\em {ACM} International Conference Proceeding Series},
  pp.~1017--1024, {ACM}, 2009.

\bibitem{vakili2024kernelbased}
S.~Vakili and J.~Olkhovskaya, ``Kernel-based function approximation for average
  reward reinforcement learning: An optimist no-regret algorithm,'' in {\em The
  Thirty-eighth Annual Conference on Neural Information Processing Systems},
  2024.

\bibitem{ZhBlBa18}
Z.~Zhou, M.~Bloem, and N.~Bambos, ``Infinite time horizon maximum causal
  entropy inverse reinforcement learning,'' {\em IEEE Transactions on Automatic
  Control}, vol.~63, no.~9, pp.~2787--2802, 2018.

\bibitem{AnKaSa25}
B.~Anahtarci, C.~D. Kariksiz, and N.~Saldi, ``Maximum causal entropy irl in
  mean-field games and gnep framework for forward rl,'' {\em Journal of Machine
  Learning Research}, vol.~26, no.~121, pp.~1--40, 2025.

\bibitem{GlTo22}
A.~Gleave and S.~Toyer, ``A primer on maximum causal entropy inverse
  reinforcement learning,'' 2022.

\bibitem{ZiBaDe10}
B.~D. Ziebart, J.~A. Bagnell, and A.~K. Dey, ``Modeling interaction via the
  principle of maximum causal entropy,'' in {\em Proceedings of the 27th
  International Conference on International Conference on Machine Learning},
  ICML'10, pp.~1255--1262, 2010.

\bibitem{ZiBaDe13}
B.~D. Ziebart, J.~A. Bagnell, and A.~K. Dey, ``The principle of maximum causal
  entropy for estimating interacting processes,'' {\em IEEE Transactions on
  Information Theory}, vol.~59, no.~4, pp.~1966--1980, 2013.

\bibitem{PaRa16}
V.~I. Paulsen and M.~Raghupathi, {\em An Introduction to the Theory of
  Reproducing Kernel Hilbert Spaces}.
\newblock Cambridge Studies in Advanced Mathematics, Cambridge University
  Press, 2016.

\bibitem{NeJoGo17}
G.~Neu, A.~Jonsson, and V.~Gomez, ``A unified view of entropy-regularized
  {M}arkov decision processes.'' arXiv preprint arXiv:1705.07798, 2017.

\bibitem{Ber99}
D.~Bertsekas, {\em Nonlinear Programming}.
\newblock Athena Scientific, 1999.

\end{thebibliography}

\end{document}